\documentclass[11pt]{article}

\usepackage[a4paper,margin=1in]{geometry}
\usepackage[T1]{fontenc}
\usepackage[utf8]{inputenc}
\usepackage{amsmath}
\usepackage{algorithm}
\usepackage{algpseudocode}
\usepackage{wrapfig}
\usepackage{adjustbox}
\usepackage[table]{xcolor}
\usepackage[authoryear,square]{natbib}
\usepackage[colorlinks=true,linkcolor=blue,citecolor=blue,urlcolor=blue]{hyperref}

\usepackage[utf8]{inputenc} % allow utf-8 input
\usepackage[T1]{fontenc}    % use 8-bit T1 fonts
\usepackage{hyperref}       % hyperlinks
\usepackage{url}            % simple URL typesetting
\usepackage{booktabs}       % professional-quality tables
\usepackage{amsfonts}       % blackboard math symbols
\usepackage{nicefrac}       % compact symbols for 1/2, etc.
\usepackage{microtype}      % microtypography
\usepackage{xcolor}         % colors
\usepackage{enumitem}
\usepackage{amsmath,amssymb,amsthm,mathtools}
\usepackage{hyperref}
\usepackage{enumitem}
\usepackage{microtype}
\usepackage{bm}

\newtheorem{definition}{Definition}
\newtheorem{assumption}{Assumption}
\newtheorem{lemma}{Lemma}
\newtheorem{proposition}{Proposition}
\newtheorem{theorem}{Theorem}

\theoremstyle{remark}
\newtheorem{remark}{Remark}

\definecolor{cMain}{RGB}{0,76,153}   % main descent term (-hA_t)
\definecolor{cBias}{RGB}{194,80,0}   % bias term (residual)
\definecolor{cMart}{RGB}{123,50,148} % martingale term
\definecolor{cSame}{RGB}{0,128,0}    % same-step correction
\definecolor{cRem}{RGB}{0,130,130}   % quadratic remainder

\newcommand{\EE}{\mathbb{E}}
\newcommand{\PP}{\mathbb{P}}
\newcommand{\calF}{\mathcal{F}}
\newcommand{\inner}[2]{\left\langle #1, #2 \right\rangle}

\newcommand{\diag}{\mathrm{diag}}

\newcommand{\epsz}{\epsilon_0}     % stabilizer
\newcommand{\yinit}{y_{\mathrm{init}}}
\newcommand{\underlineD}{\underline D}
\newcommand{\overlineD}{\overline D}
\newcommand{\clip}{\mathrm{clip}}
\newcommand{\one}{\mathbf 1}
\definecolor{greyrow}{RGB}{217,217,217}
\definecolor{lightgreyrow}{RGB}{242,242,242}
\definecolor{lightgreenrow}{RGB}{226,239,218}
\title{Why Clipping Matters in AdaGrad?\\ Toward a High-Probability Theory under Generalized Smoothness}

\author{
  Alokendu Mazumder$^{1}$\quad Ayaan Mohd$^{3}$ \quad Harshit Rawat$^{4}$ \quad Arnab Roy$^{5,1}$ \\ \quad Mayank Baranwal$^{6,7}$\quad Punit Rathore$^{1,2}$ \\
  \small $^1$ Robert Bosch Center for Cyber Physical Systems, IISc Bengaluru \\
  \small $^2$ Centre for Infrastructure, Sustainable Transportation and Urban Planning, IISc Bengaluru  \\
  \small $^3$ Undergraduate Program, IISc Bengaluru\\
  \small $^4$ Department of Computational Data Sciences, IISc Bengaluru\\
  \small $^5$ Department of Computer Science and Automation, IISc Bengaluru\\
  \small $^6$ Tata Consultancy Services Research, Mumbai\\
  \small $^7$ System and Control Engineering, IIT Bombay\\
  \small \texttt{\{alokendum, ayaanmohd, harshitrawat, arnabroy, prathore\}@iisc.ac.in}\\
  \small \texttt{baranwal.mayank@tcs.com}, \texttt{mbaranwal@iitb.ac.in}
}

\begin{document}

\maketitle

\begin{abstract}
We analyze the original same-step coordinate-wise AdaGrad under generalized smoothness and heavy-tailed noise with bounded variance. In this setting, local curvature may grow sub-quadratically with the gradient norm, and stochastic gradients are assumed to have only bounded conditional second moments. We show that unclipped AdaGrad can become \emph{anisotropically miscalibrated}: under heavy-tailed noise, the adaptive denominator can learn the geometry of rare noise shocks rather than the local curvature of the objective, leading to a persistent directional distortion that blocks finite-horizon Euclidean progress. We then prove that clipping repairs this failure mode. Our main result is a finite-horizon high-probability guarantee for the original non-lagged AdaGrad update, yielding $\frac1T\sum_{t=0}^{T-1}\|\nabla f(x_t)\|^2=\mathcal{O}\left(\frac{d\big(\sqrt{\log T} + \log \frac{1}{\delta}\big)}{\sqrt{T}}\right),$ and hence $\widetilde{\mathcal O}(\varepsilon^{-2})$ complexity. This shows that, for AdaGrad under heavy-tailed noise, clipping is a structural stabilizer of the adaptive geometry rather than merely a robustness heuristic.
\end{abstract}
\tableofcontents

\section{Introduction}
\label{sec:intro}

In this paper, we study the following optimization problem
\begin{equation}
\label{eq:1}
    \min_{x\in\mathcal{X}} f(x),
\end{equation}
where $\mathcal{X} \subseteq\mathbb{R}^d$ is the domain of $f$. Classical textbook analyses~\citep{nemirovskij1983problem,nesterov2013introductory} of~(\ref{eq:1}) often require the Lipschitz smoothness condition, which assumes $\|\nabla^2 f(x)\| \le L$ almost everywhere for some $L\geq 0$ called the smoothness constant. This condition is restrictive: it imposes a single global curvature scale and implies that the objective is locally controlled above and below by quadratic models.%This condition, however, is rather restrictive and only satisfied by functions that are both upper and lower bounded by quadratic functions.
% The modern theory of first-order optimization is still built, to a remarkable extent, around one global geometric premise: the gradient is Lipschitz continuous, or equivalently, the Hessian is uniformly bounded by a constant. This assumption is mathematically convenient and has powered a vast literature, but it is also narrow. 
It effectively restricts the objective to a geometry whose curvature is globally controlled by a single number, and therefore excludes many functions whose local curvature naturally changes with the scale of the gradient. In problems arising from modern machine learning, this rigidity is often difficult to justify~\citep{zhang2019gradient,faw2023beyond,wang2024provable,vankov2025optimizing}. The geometry seen by the iterates is typically far from uniform, and the relevant question is not whether the objective is globally well-approximated by a quadratic everywhere, but whether it admits a useful \emph{local} smoothness description along the optimization path. 

% Recently, \citet{NEURIPS2023_7e8bb8d1} proposed a generalized smoothness framework that relaxes the classical global Lipschitz-gradient assumption. Their framework allows curvature to vary with the local gradient scale through a condition of the form
% \[
% \|\nabla^2 f(x)\| \le \ell(\|\nabla f(x)\|),
% \]
% where $\ell$ is a nondecreasing function, typically assumed to be sub-quadratic in the regime of interest. This substantially enlarges the admissible function class beyond globally smooth objectives and captures settings in which curvature grows with gradient magnitude. More importantly, it changes the logic of the analysis. Under this model, one does not assume bounded gradients as a primitive condition. Rather, the point of generalized smoothness is that gradient control can be \emph{derived} from trajectory control itself. If the iterates remain in a bounded sublevel set, then generalized smoothness implies a corresponding bound on the gradient magnitude, which in turn induces an effective local smoothness constant \(L=\ell(2G)\) along the trajectory. Once this happens, the objective behaves locally like a classically smooth function on the region explored by the iterates. In this sense, generalized smoothness replaces a fixed global geometric description by a \emph{trajectory-dependent} one.
Recently, \citet{NEURIPS2023_7e8bb8d1} proposed a generalized smoothness framework that relaxes global Lipschitz smoothness by allowing curvature to scale with the gradient, $\|\nabla^2 f(x)\| \le \ell(\|\nabla f(x)\|)$, where \(\ell\) is nondecreasing and typically sub-quadratic. This framework covers objectives beyond globally smooth ones and includes, as special or related cases, several relaxed smoothness models such as $(L_0,L_1)$-smoothness~\citep{zhang2019gradient}. More importantly, it changes the logic of the analysis: rather than assuming a fixed global curvature bound, one controls the trajectory, derives gradient bounds from sublevel-set control, and thereby obtains an effective local smoothness constant along the path. In this sense, generalized smoothness replaces a static global geometry by a \emph{trajectory-dependent} one.
% Under this model, one no longer begins with a fixed global curvature constant and derives descent uniformly over the whole space. Instead, one first seeks to control the trajectory itself: if the gradients along the optimization path remain bounded by a constant $G>0$, then this induces an \emph{effective local smoothness constant} $L = \ell(2G)$, and the objective behaves locally like a classically smooth function. In this sense, the generalized smoothness framework replaces a global geometric description by a \emph{trajectory-dependent} one.

This trajectory-dependent perspective makes adaptive methods such as AdaGrad~\citep{streeter2010less,duchi2011adaptive} and Adam~\citep{kingma2015Adam} especially natural for modern machine-learning objectives~\citep{vaswani2017attention,you2019large,nikishina2022cross,li2022sp2,abdukhakimov2024stochastic,li2024enhancing,schaipp2023momo,loizou2021stochastic,moskvoretskii2024large,shi2021ai}. If the local geometry changes with the gradient scale, then algorithms that rescale their steps according to observed gradient magnitudes should, at least in principle, be better matched to the landscape than fixed-step methods. AdaGrad is the canonical example: it accumulates coordinate-wise squared gradients and uses them to define an adaptive diagonal metric. In benign stochastic regimes, this mechanism can be interpreted as learning a useful coordinate-wise scale for the problem. Under heavy-tailed noise, however, this interpretation becomes questionable. The accumulator does not observe curvature directly; it observes squared stochastic gradients. Our results show that rare extreme realizations can dominate particular coordinates, causing the learned metric to reflect the noise anisotropy over the objective anisotropy.
\begin{algorithm}[H]
\caption{AdaGrad~\citep{duchi2011adaptive} (Clipped and Unclipped)}
\label{alg:non_lagged_adagrad}
\begin{algorithmic}[1]
\State \textbf{Input:} Initial $x_0 \in \mathbb{R}^d$, accumulator $y_0 = y_{\mathrm{init}}\mathbf{1}$, stepsize $h > 0$, parameter $\epsilon_0 > 0$, clipping threshold $C > 0$
\For{$t = 0, 1, \dots, T-1$}
    \State $g_t = \nabla f(x_t) + \xi_t$
    \State $y_{t+1} \coloneqq y_t + g_t^{\odot 2}$
    \State $D_{t+1} \coloneqq \operatorname{diag}(\sqrt{y_{t+1}} + \epsilon_0)$
    \State $x_{t+1} \coloneqq x_t - h D_{t+1}^{-1} g_t$
\EndFor
\end{algorithmic}
\end{algorithm}
%We hypothesize that if rare extreme realizations dominate some coordinates, the learned metric may reflect the noise's anisotropy rather than the objective's anisotropy.

This issue is not visible in the classical smooth theory at the level we need here. Existing lower-bound results for AdaGrad already reveal several distinct pathologies, but they do so from perspectives different from ours. In the heavy-tailed setting, \citet{chezhegov2025clipping} show that \emph{AdaGrad-Norm} and \emph{Adam-Norm} (with and without delay) can have provably poor high-probability convergence, and that gradient clipping restores polylogarithmic dependence on the confidence level. In the relaxed-smoothness setting, \citet{crawshaw2025complexity} prove complexity lower bounds for several AdaGrad variants under \((L_0,L_1)\)-smoothness~\citep{zhang2019gradient}, showing that the dependence on \(\Delta,L_0,L_1\) can be intrinsically worse than in the globally smooth case. Under a different refined anisotropic model, \citet{pmlr-v291-jiang25c} provide supporting lower bounds specific to AdaGrad to show tightness of their upper bounds and to separate AdaGrad from SGD. These results are important, but they leave open a different question that becomes central in the generalized-smoothness regime studied here: 
\begin{quote}
    \centering \itshape
        What geometric object is the unclipped adaptive metric actually learning?
\end{quote}
Since generalized-smoothness descent is local and path-dependent, it is not enough for the adaptive denominator to merely stay finite; it must encode the \emph{right directional scale}. Our point is that, without clipping, this can fail in a sharper way than mere slowdown: the adaptive metric can become dominated by anisotropy in the noise rather than anisotropy in the objective.

This leads to the central questions of the paper:

\begin{quote}
    \centering \itshape
    Can unclipped AdaGrad preserve the correct adaptive geometry under generalized smoothness and heavy-tailed noise? If not, can clipping repair this failure and yield finite-horizon high-probability guarantees?
    %Can unclipped AdaGrad provably preserve the correct adaptive geometry under generalized smoothness and heavy-tailed noise? If not, does clipping repair precisely this failure and allow a finite-horizon high-probability theory?
\end{quote}

Our answer is negative to the first question and positive to the second. 

\subsection{Our Contributions}
\label{sec:contributions}

The main contributions of this work are summarized below. We emphasize that the principal contribution of this paper is theoretical: we provide an in-depth analysis of clipping mechanisms within established adaptive methods like AdaGrad, without aiming to present an empirical study.

\begin{itemize}[leftmargin=*]
    \item \textbf{Negative Result for AdaGrad (Proposition~\ref{thm:anisotropic_miscalibration}).}  
    We prove a finite-horizon lower bound showing that, under anisotropic heavy-tailed noise, AdaGrad can mislearn the problem geometry. Specifically, we construct a convex stochastic optimization problem with bounded conditional second moments in which a single rare shock distorts the adaptive denominator in the wrong coordinate and keeps the Euclidean stationarity average bounded away from zero.
    % We prove a finite-horizon lower bound showing that under anisotropic heavy-tailed noise, AdaGrad update can learn the wrong directional scale: the adaptive denominator becomes dominated by the noise geometry, and the Euclidean stationarity average stays bounded away from zero. In particular, we design an example of convex stochastic optimization problem such that the noise is heavy tailed and....

    \item \textbf{Clipping fixes AdaGrad and New Upper Bounds (Theorem~\ref{thm:main}).} We show that gradient clipping repairs this failure mode. Under sub-quadratic generalized smoothness and heavy-tailed noise with bounded conditional second moments, we prove a finite-horizon high-probability convergence guarantee for the original clipped AdaGrad update, yielding a $\widetilde{\mathcal{O}}(\varepsilon^{-2})$ complexity bound with polylogarithmic dependence on the failure probability. %We prove that the above issue can be addressed via gradient clipping. Concretely, we develop a finite-horizon high-probability theory for the original clipped AdaGrad update under sub-quadratic generalized smoothness and heavy-tailed noise with bounded conditional second moments, yielding a high-probability $\tilde{\mathcal{O}}(\varepsilon^{-2})$ complexity bound. 

    \item \textbf{Analytical Parameter Regimes (Section~\ref{sec:param}).} We provide explicit, theoretically constructible bounds for the stepsize, clipping level, and warm start. This shows that true $\varepsilon$-target Euclidean stationarity is theoretically attainable under a strict, mathematically defined regime, solidifying the high-probability bounds into a complete non-asymptotic guarantee.
    % \item \textbf{A high-probability theory for clipped AdaGrad.}
    % We establish Euclidean stationarity guarantees for clipped AdaGrad under sub-quadratic generalized smoothness and heavy-tailed noise with bounded conditional second moments, yielding a high-probability \(\widetilde{\mathcal O}(\varepsilon^{-2})\) complexity bound.

    \item \textbf{Novel Proof Architecture for Same-Step Updates (Lemma~\ref{lem:samestep}).}
A major theoretical hurdle in analyzing AdaGrad is the correlation between the current gradient and the concurrent adaptive denominator $D_{t+1}^{-1}$ \citep{wang2023convergence,faw2022power}. We resolve this by cleanly decoupling the historical metric from the same-step correction, bounding the latter via a deterministic telescoping argument.
    
    % \item \textbf{A Proof Architecture for the Original Same-Step Update (Lemma~\ref{lem:samestep}).}
    % We overcome one of the core analytical complications of same-step dependence in AdaGrad, the correlation between the current gradient and the adaptive denominator $D_{t+1}^{-1}$ \citep{wang2023convergence,faw2022power}. We handle this by separating the past metric from the same-step correction and controlling the latter through a deterministic telescoping argument.

%         \item \textbf{An explicit finite-horizon \(\varepsilon\)-target regime.}
% We show that with a suitable calibration of the clipping threshold, warm start, and stepsize, clipped AdaGrad attains a true finite-horizon Euclidean guarantee with \(\widetilde{\mathcal O}(\varepsilon^{-2})\) complexity. This turns the high-probability theory into a concrete and usable parameter regime rather than a purely qualitative convergence statement.
\end{itemize}

Most stochastic optimization guarantees are stated in expectation (e.g., $\mathbb{E}\|\nabla f(x)\|^2\le \varepsilon$), obscuring the algorithm's pathwise behavior. Markov conversions give high probability only by requiring expected accuracy $\varepsilon\delta$; since standard nonconvex complexities scale as $\mathcal{O}(\varepsilon^{-2})$, this leads to a prohibitive $\mathcal{O}(\varepsilon^{-2}\delta^{-2})$ dependence. A meaningful high-probability theory should instead incur only polylogarithmic dependence on $1/\delta$. This is challenging for adaptive methods because stepsizes are random and require pathwise control \citep{li2020high,attia2023sgd,chezhegov2025clipping}. In AdaGrad, the main obstacle is \emph{same-step dependence}: the metric $D_{t+1}^{-1}$ uses the same stochastic gradient as the update. Consequently, many prior analyses use delayed variants to recover conditional independence \citep{li2019convergence,li2020high,faw2022power,savarese2021domain,chakrabarti2024methodology,chezhegov2025clipping}. We analyze the original non-lagged update directly. Gradient clipping controls heavy-tailed noise, while a deterministic telescoping argument isolates the same-step correction, yielding a genuine $\widetilde{\mathcal{O}}(\varepsilon^{-2}\log(1/\delta))$ complexity bound without delays or average-case reductions.

\subsection{Preliminaries}
\label{sec:prelim}

In this section, we define the noise model and the class of functions to which the objective $f$ belongs. 
% We begin with a standard assumption in unconstrained optimization. Unless stated otherwise, these apply throughout the remaining sections.
\begin{assumption}
\label{ass:1}
    The objective function $f$ is differentiable and closed within its open domain $\text{dom}(f) \subseteq \mathbb{R}^d$ and is bounded from below, i.e., $f^* \coloneqq \inf_x f(x) > -\infty$.
\end{assumption}

\subsubsection{Generalized Smoothness}
We now formally introduce the generalized smoothness condition~\cite{NEURIPS2023_7e8bb8d1}, and present its properties.

\begin{definition}[$\ell$-smoothness]
\label{def:1}
A real-valued differentiable function $f : \mathcal{X} \to \mathbb{R}$ is $\ell$-smooth for some non-decreasing continuous function $\ell : [0, +\infty) \to (0, +\infty)$ if
$$
\|\nabla^2 f(x)\| \le \ell(\|\nabla f(x)\|)
$$
almost everywhere (with respect to the Lebesgue measure) in $\mathcal{X}$.
\end{definition}

\begin{remark}
Definition 1 reduces to the classical $L$-smoothness when $\ell \equiv L$ is a constant function. It reduces to the $(L_0, L_1)$-smoothness proposed in \cite{zhang2019gradient} when $\ell(u) = L_0 + L_1 u$ is an affine function.
\end{remark}

First, we provide the following lemma, which is very useful in our analyses in this paper.

\begin{lemma}[\cite{NEURIPS2023_7e8bb8d1}]
\label{lem:gen-smooth}
If $f$ is $\ell$-smooth, for any $x \in \mathcal{X}$ satisfying $\|\nabla f(x)\| \le G$, we have (1) $\mathcal{B}(x, G/L) \subseteq \mathcal{X}$, and (2) for any $x_1, x_2 \in \mathcal{B}(x, G/L)$,
    \begin{equation}
    \label{eq:lips}
    \|\nabla f(x_1) - \nabla f(x_2)\| \le L \|x_1 - x_2\|, \quad f(x_1) \le f(x_2) + \langle \nabla f(x_2), x_1 - x_2 \rangle + \frac{L}{2} \|x_1 - x_2\|^2,
    \end{equation}
where $L \coloneqq \ell(2G)$ and $\mathcal{B}(x,R)$ denotes euclidean ball with radius $R$ and center at $x$.
\end{lemma}

\textbf{Lemma}~\ref{lem:gen-smooth} states that, if the gradient at $x$ is bounded by some constant $G$, then within its neighborhood with a radius $G/L$, we can obtain (\ref{eq:lips}), the same inequalities that were derived in the textbook analyses \citep{nesterov2013introductory} under the standard Lipschitz smoothness condition. With (\ref{eq:lips}), the analysis for generalized smoothness is not much harder than that for standard smoothness. Since we mostly choose $x = x_2 = x_t$ and $x_1 = x_{t+1}$ in the analysis, in order to apply \textbf{Lemma}~\ref{lem:gen-smooth}, we need two conditions: $\|\nabla f(x_t)\| \le G$ and $\|x_{t+1} - x_t\| \le G/L$ for some constant $G$. The latter is usually directly implied by the former for most deterministic methods with a sufficiently small step size; furthermore, the former can be obtained via the approach by \citet{NEURIPS2023_7e8bb8d1}, which bounds the gradients along the trajectory. We show it in the following lemma.

\begin{lemma}[\citet{NEURIPS2023_7e8bb8d1}]
\label{lem:levelset}
Assume $f$ is $\ell$-smooth with sub-quadratic $\ell$.
If $f(x)-f^\star \le F$ for some $x \in \mathcal{X}$ and $F\geq 0$, define $
G \coloneqq \sup\Big\{u\ge 0:\ u^2 \le 2\,\ell(2u)\,F\Big\},$ $L \coloneqq \ell(2G),$
then they satisfy $G^2 = 2\ell(2G)F$ and we have $G<\infty$ and $\|\nabla f(x)\|\le G$.
\end{lemma}

Therefore, in order to bound the gradients along the trajectory as we discussed in \textbf{Lemma}~\ref{lem:levelset}, it
suffices to bound the function values, which is usually easier.

\begin{assumption}
    \label{ass:gen}
    The objective function $f$ is $\ell$-smooth with subquadratic $\ell$.
    \end{assumption}

The standard smooth function class is very restrictive as it only contains functions that are upper and
lower bounded by quadratic functions. The $(L_0, L_1)$ smooth function class~\citet{zhang2019gradient} is more general since it also
contains, e.g., univariate polynomials and exponential functions. Assumption~\ref{ass:gen} is even more general and
contains univariate rational functions, double exponential functions, etc. Next, we make a standard assumption on the gradient noise model. 

At each iteration $t$, we observe a stochastic gradient $g_t = \nabla f(x_t) + \xi_t$. Here $\xi$ is a random variable
following some distribution that may be dependent on $x$
and time. We define the filtration $\mathcal{F}_t \coloneqq \sigma(g_0, \dots, g_t)$ and denote the conditional expectation as $\mathbb{E}_{t-1}[\cdot] \coloneqq \mathbb{E}[\cdot | \mathcal{F}_{t-1}]$. $g_t$ might be unbounded due to potentially unbounded gradient noise. 
% The clipping operator with clipping level $C > 0$ is defined as $\text{clip}(x, C) \coloneqq \min\left\{1, \frac{C}{\|x\|}\right\}x$ for $x \neq 0$ and $\text{clip}(x, C) \coloneqq 0$ for $x = 0$.

 \begin{assumption}[\cite{nemirovski2009robust,ghadimi2012optimal,takac2013mini,NEURIPS2023_7e8bb8d1}]
 \label{ass:noise}
     $\mathbb{E}_{t-1}[\xi_t] = 0$ and $\mathbb{E}_{t-1}[\|\xi_t\|^2] \le \sigma^2$ for some $\sigma \ge 0$.
 \end{assumption}

% The above assumption is standard and used in many prior works~\citep{}.

% Existing lower bounds for unclipped Adam/AdaGrad-type methods under heavy-tailed noise, such as \citet{chezhegov2025clipping}, establish poor high-probability behavior in globally smooth settings and show that clipping repairs this failure. Those results are important, but they do not identify what geometric object the unclipped adaptive metric is actually learning. In the generalized-smoothness regime studied in this paper, this question becomes central: since descent is local and geometry is path-dependent, the adaptive metric must not merely stay finite; it must encode the \emph{right directional scale}. Our claim is that, without \emph{clipping}, this can fail in a sharper way than mere slowdown.
\section{Anistropic Miscalibration and The Failure of AdaGrad}
\label{sec:LB}
We demonstrate a negative convergence result for AdaGrad (\textbf{Algorithm}~\ref{alg:non_lagged_adagrad}) caused by \emph{anisotropic miscalibration}. Heavy-tailed noise in flat directions can dominate the adaptive metric, severely under-updating steeper directions and misaligning with local geometry.

\begin{proposition}[Anisotropic miscalibration lower bound for AdaGrad]
\label{thm:anisotropic_miscalibration}
Let
\[
    f(x_1,x_2)=\frac14 x_1^4+\frac12 x_1^2+8x_2^2
\]
and \(x_0=(2,1)\). Run \textbf{Algorithm}~\ref{alg:non_lagged_adagrad}
with \(d=2\), \(h=1\), \(y_{\mathrm{init}}=0\), and \(\epsilon_0=1\).
For every horizon \(T\ge 2\) and every noise scale \(\sigma>0\), there exists
a stochastic oracle satisfying Assumption~\ref{ass:noise} with noise level
\(\sigma\) such that, with probability at least $p_T\coloneqq\frac{\sigma^2}
    {\bigl(32(T-1)-10\bigr)^2+\sigma^2},$
the following inequalities hold simultaneously for every \(t=1,\dots,T-1\):
\[
    \partial_{11}^2 f(x_t)<\partial_{22}^2 f(x_t),\qquad
    (D_t)_{11}>\frac{16}{9}(T-1)(D_t)_{22},\qquad
    \|\nabla f(x_t)\|>1.
\]
Consequently, since \(\|\nabla f(x_0)\|>1\), we have $\mathbb P\!\left(
    \frac1T\sum_{t=0}^{T-1}\|\nabla f(x_t)\|^2>1
    \right)\ge p_T$. Therefore, any uniform guarantee of the form $\mathbb P\!\left(
    \frac1T\sum_{t=0}^{T-1}\|\nabla f(x_t)\|^2\le 1
    \right)\ge 1-\delta$ over all stochastic oracles satisfying Assumption~\ref{ass:noise} requires $T\ge
    1+\frac{1}{32}
    \left(
        10+\sigma\sqrt{\frac{1-\delta}{\delta}}
    \right)$.
    
\end{proposition}

% \begin{proposition}[Anisotropic miscalibration lower bound for AdaGrad]
% \label{thm:anisotropic_miscalibration}
% Let 
% \[
% f(x_1,x_2)=\frac14 x_1^4+\frac12 x_1^2+8x_2^2
% \]
% and \(x_0=(2,1)\). Run \textbf{Algorithm}~\ref{alg:non_lagged_adagrad} with \(d=2\), \(h=1\), \(y_{\mathrm{init}}=0\), and \(\epsilon_0=1\). For every horizon \(T\ge 2\) and every \(\sigma>0\), there exists a stochastic oracle satisfying Assumption~\ref{ass:noise} such that, with probability at least $p_T\coloneqq\frac{\sigma^2}{(32(T-1)-10)^2+\sigma^2},$ one has for every \(t=1,\dots,T-1\),
% \[
% \partial_{11}^2 f(x_t)<\partial_{22}^2 f(x_t),\qquad
% (D_t)_{11}>\frac{16}{9}(T-1)(D_t)_{22},\qquad
% \|\nabla f(x_t)\|>1.
% \]
% Consequently, $\mathbb P\!\left(
% \frac1T\sum_{t=0}^{T-1}\|\nabla f(x_t)\|^2>1
% \right)\ge p_T.$ Hence, any guarantee of the form
% $\mathbb P\!\left(
% \frac1T\sum_{t=0}^{T-1}\|\nabla f(x_t)\|^2\le 1
% \right)\ge 1-\delta
% $
% requires $T\ge 1+\frac{1}{32}\!\left(10+\sigma\sqrt{\frac{1-\delta}{\delta}}\right).
% $
% \end{proposition}

\begin{proof}[Proof Sketch]
We construct a heavy-tailed oracle triggering a shock $\xi_0 = (B, 0)$ with probability $p_T = \mathcal{O}(\sigma^2/B^2)$, permanently inflating the adaptive metric to $(D_t)_{11} = \Omega(T)$. Under noiseless subsequent updates, the second coordinate contracts exponentially while the penalized first coordinate stalls ($x_{t,1} > 7/8$), establishing a strict anisotropic mismatch $(D_t)_{11} \gg (D_t)_{22}$. This learned metric inversion prevents Euclidean convergence, bounding the gradient norm $\|\nabla f(x_t)\| > 1$ for all $t < T$. Forcing the stationarity failure probability $p_T \le \delta$ algebraically isolates $B$ and mandates the finite-horizon lower bound $T = \Omega\bigl(\sigma\sqrt{1/\delta}\bigr)$. \emph{Proof is deferred to Appendix~\ref{sec:ams}.}
% We construct a heavy-tailed oracle triggering an initial shock \(\xi_0 = (B, 0)\) with probability \(p_T = \mathcal{O}(\sigma^2/B^2)\), permanently inflating the adaptive metric to \((D_t)_{11} = \Omega(T)\). Under noiseless subsequent updates, the second coordinate contracts exponentially while the artificially penalized first coordinate stalls (\(x_{t,1} > 7/8\)), establishing a strict anisotropic mismatch \((D_t)_{11} \gg (D_t)_{22}\). This learned metric inversion prevents Euclidean convergence, strictly bounding the gradient norm \(\|\nabla f(x_t)\| > 1\) for all \(t < T\). Forcing the stationarity failure probability \(p_T \le \delta\) therefore algebraically isolates \(B\) and mandates the finite-horizon lower bound \(T = \Omega\bigl(\sigma\sqrt{1/\delta}\bigr)\). \emph{Full proof is deferred to Appendix~\ref{sec:ams}.}
\end{proof}

While \cite[Theorem~2.1]{chezhegov2025clipping} shows that heavy-tailed noise slows AdaGrad's concentration, Proposition~\ref{thm:anisotropic_miscalibration} reveals a more fundamental obstruction: an irreversible collapse of the algorithm's geometry. This finite-horizon failure leaves the average gradient norm strictly bounded away from zero, establishing the structural necessity of clipping. Rather than a mere robustness heuristic, clipping is the exact mechanism required to prevent rare shocks from permanently poisoning the adaptive denominator. Therefore, our contributions follow a clear trajectory: we formalize this structural failure of plain AdaGrad, prove that clipping systematically removes the obstruction, and develop a novel proof architecture to analyze the exact same-step update.
\section{Related Work}

\begin{table}[ht]
\caption{
Summary of existing high probability convergence results for AdaGrad.
\textbf{None of these works show}
$
\frac1T\sum_{t=1}^T \|\nabla f(x_t)\|^2
= \tilde{\mathcal O}(T^{-1/2})
$
\textbf{for clipped AdaGrad under generalized smoothness.} (1) Uses $L_1$ average gradient norm stationarity. (2) They assume the objective function is uniformly upper-bounded. $\$$ denotes that the rates are given for $\alpha = 2$.} 
\label{tab:my-table}
\resizebox{\columnwidth}{!}{%
\begin{tabular}{lcccccl}
\toprule
 &
  \begin{tabular}[c]{@{}c@{}}Allow \\unbounded\\ gradient\end{tabular} &
  Smooth &
  Noise &
  Rate &
  Conv. &
  Comments \\ \midrule
Ward et al. [2020]~\citep{ward2020adagrad} &
  $\times$ &
  $L$ &
  bounded variance &
  $O(\frac{\log T}{\delta^{3/2}\sqrt{T}} + \frac{\log^2T}{\delta^2 T})$ &
 best iterate &
  AdaGrad Norm \\
Li et al. [2020]~\cite{li2020high} &
  $\checkmark$ &
  $L$ &
  sub-Gaussian &
  $O\Big(\frac{d\log^{3/2}(T/\delta)}{\sqrt{T}} + \frac{d^2\log^2(T/\delta)}{T}\Big)$ &
 best iterate &
  Delayed AdaGrad with momentum \\
Kavis et al. [2022]~\cite{kavis2022high} &
  $\times$ &
  $L$ &
  bounded variance&
  $O(\frac{\log T + \log (1/\delta) + \sqrt{\log(1/\delta)}}{\sqrt T})$ & average stationarity&
  AdaGrad-Norm \\
Attia et al. [2023]~\cite{attia2023sgd} &
  $\checkmark$ &
  $L$ &affine noise&
 $O\left( \frac{\log^2(T/\delta)}{\sqrt{T}} + \frac{\log^4(T/\delta) \log(1/\delta)}{T} \right)$ &
  average stationarity &
  SGD with AdaGrad stepsize\\
Liu et al. [2023]~\cite{liu2023high} &
  $\checkmark$ &
  $L$ &sub-Gaussian&
 $\mathcal{O}\left( \frac{d^3 \log^{3/2}(dT/\delta)}{\sqrt{T}} + \frac{d^3 \log^{3/2}(dT/\delta) \log(d/\delta)}{T} \right)$ &
  average stationarity$^{(1)}$ &
  AdaGrad\\
Li et al. [2023]~\cite{li2023high} &
  $\checkmark^{(2)}$ &
  $L$ &bounded non-central $\alpha$ moment of stoc. grad.&
 $\mathcal{O}\Big(\frac{\log(1/\delta)}{\sqrt T}\Big)^{\$}$ &
  average stationarity &
  Clipped AdaGrad-Norm\\
Zhou et al. [2024]~\cite{zhou2024on} &
  $\times$ &
  $L$ & sub-Gaussian&
  $O(\sqrt{d/T}\log(1/\delta)+d/T)$ &
  average stationarity &
  AdaGrad\\  
Chezhegov et al. [2026]~\cite{chezhegov2025clipping} &
  $\checkmark$ &
  $L$ &
  \begin{tabular}[c]{@{}c@{}}coordinate wise heavy-tail \\ ($\alpha$-moment)\end{tabular} &
  $O\Big(\frac{d^{2/3}\log^{2/3}(T/\delta)}{\sqrt{T}}\Big)^{\$}$ &
  average stationarity&
  Clipped Delayed AdaGrad\\
\rowcolor{greyrow}
\textcolor{blue}{This Work} &
  $\checkmark$ &
  $\ell$ (sub-quad.) &
  bounded variance&
  $\mathcal{O}\left( \frac{d\big(\sqrt{\log T}+ \log(1/\delta)\big)}{\sqrt{T}} \right) $ &
  average stationarity &
  Clipped AdaGrad\\ \bottomrule
\end{tabular}%
}
\end{table}
\textbf{Adaptive Methods and Classical Smoothness.} Adaptive methods like AdaGrad~\cite{duchi2011adaptive} and Adam~\citep{kingma2015Adam} anchor stochastic optimization. Extensive literature analyzes variants (e.g., AMSGrad, Yogi, Padam, AdaBound) to improve convergence guarantees and practical behavior~\cite{reddi2019convergence, zaheer2018adaptive, chen2018closing, luo2019adaptive}. While their convergence theory traditionally assumes standard $L$-smoothness, studies establish sharp nonconvex guarantees for AdaGrad under relaxed assumptions~\cite{ward2020adagrad, wang2023convergence, hong2024revisiting}. Similarly, Adam's convergence has been refined under various stochastic conditions~\cite{reddi2019convergence, zou2019sufficient, defossez2020simple, hong2024convergence}.\\
% Adaptive methods like AdaGrad~\cite{duchi2011adaptive} and Adam~\citep{kingma2015Adam} form the backbone of stochastic optimization. A vast literature analyzes their variants (e.g., AMSGrad, Yogi, Padam, AdaBound) to improve convergence guarantees and practical behavior~\cite{reddi2019convergence, zaheer2018adaptive, chen2018closing, luo2019adaptive}. Convergence theory for these methods traditionally assumes standard $L$-smoothness. For AdaGrad, studies have established sharp nonconvex guarantees under relaxed assumptions~\cite{ward2020adagrad, wang2023convergence, hong2024revisiting}. Similarly, Adam's convergence has been extensively refined and corrected under various stochastic conditions~\cite{reddi2019convergence, zou2019sufficient, defossez2020simple, hong2024convergence}.
\textbf{Generalized Smoothness and Relaxed Curvature.}
Recent works replace global smoothness with generalized conditions where local smoothness scales with the gradient norm. Introduced via $(L_0,L_1)$-smoothness~\cite{zhang2019gradient}, this framework has been expanded to sub-quadratic generalized smoothness and variance-reduced methods~\cite{NEURIPS2023_7e8bb8d1, chen2023generalized}, as well as parameter-agnostic and mirror-descent algorithms~\cite{hubler2024parameter, vankov2025optimizing, yu2025mirror}. Under these relaxed conditions, AdaGrad's convergence has been analyzed~\cite{wang2023convergence}, with evidence that coordinate-wise adaptivity exploits directional curvature well~\cite{liu2024adagrad}. Adam has similarly been proven to effectively exploit non-uniform smoothness~\cite{li2023convergence, wang2024provable, wang2024convergence, crawshaw2022robustness, hong2024convergence}.\\
\textbf{Gradient Clipping and Heavy-Tailed Noise.}
Gradient clipping is widely recognized as a robust stabilizer against non-uniform curvature and rapidly growing objectives~\cite{zhang2019gradient, mai2021stability}, including in smooth convex optimization with heavy-tailed noise~\cite{gorbunov2020stochastic}. High-probability guarantees under heavy-tailed noise further demonstrate that combining clipping with momentum or normalization stabilizes trajectories~\cite{cutkosky2021high, hubler2024gradient, fatkhullin2025can, hubler2024parameter}. Most relevant to our work, \cite{chezhegov2025clipping} shows that clipping repairs the vulnerability of AdaGrad/Adam-Norm to heavy-tailed shocks, yielding polylogarithmic confidence dependence.

\section{New Upper Bounds}
\label{sec:upper}
The theoretical analysis in this section is therefore structured to prevent 
\emph{anisotropic miscalibration}, all while retaining the exact \emph{same-step} AdaGrad update. The point is not simply that clipping reduces variance. In the present generalized-smoothness regime, descent is local and path-dependent, so the adaptive denominator must
encode a \emph{usable} directional scale along the entire trajectory. Once gradients are clipped, the only
new difficulty in the original non-lagged AdaGrad update is the \emph{same-step coupling} between the
metric \(D_{t+1}^{-1}\) and the current clipped gradient \(\tilde g_t\) as mentioned in Section~\ref{sec:contributions}. To be precise, at iteration \(t\), predictability is understood with respect to the filtration
\(\mathcal F_t \coloneqq \sigma(g_0,\dots,g_t)\). Since \(D_t\) depends only on
\(\{\tilde g_0,\dots,\tilde g_{t-1}\}\), the matrix \(D_t^{-1}\) is \(\mathcal F_{t-1}\)-measurable.
In contrast, \(D_{t+1}^{-1}\) depends on the current clipped gradient \(\tilde g_t\) through
\(y_{t+1}=y_t+\tilde g_t^{\odot 2}\), and is therefore only \(\mathcal F_t\)-measurable. Hence \(D_{t+1}^{-1}\) is not predictable. 

\textbf{Clipped Noise:} We decompose the clipped gradient $\tilde g_t$ into a predictable mean and a martingale part. Define the clipping residual \(r_t\coloneqq\tilde g_t\!-\!g_t\) and the predictable mean \(\mu_t \coloneqq \mathbb{E}_{t-1}[\tilde g_t] = \nabla f(x_t) + \mathbb{E}_{t-1}[r_t]\). Then the clipped noise, \(\tilde \xi_t\!\coloneqq\!\tilde g_t\!-\!\mu_t\), satisfies \(\mathbb{E}_{t-1}[\tilde \xi_t]\!=\!0, \|\tilde \xi_t\| \le 2C\).

Furthermore, we show that this dependence can be exactly isolated as a deterministic correction, pathwise telescoped, and absorbed through a matched warm start combined with a cubic clipping calibration. This establishes clipping not as a generic robustness heuristic, but as a structural stabilizer of the adaptive geometry, consistent with its recent applications in high-probability optimization under heavy-tailed noise \citep{gorbunov2020stochastic,cutkosky2021high,chezhegov2025clipping}. Methodologically, our generalized-smoothness localization adopts the trajectory-control perspective of \cite{NEURIPS2023_7e8bb8d1}, while successfully navigating the same-step dependence that has notoriously complicated prior AdaGrad analyses \citep{wang2023convergence,faw2022power,li2020high}.

\begin{algorithm}[H]
\caption{Clipped AdaGrad}
\label{alg:clipped_adagrad}
\begin{algorithmic}[1]
\State \textbf{Input:} Initial $x_0 \in \mathbb{R}^d$, accumulator $y_0 = y_{\mathrm{init}}\mathbf{1}$, stepsize $h > 0$, parameter $\epsilon_0 > 0$, clipping threshold $C > 0$.
\For{$t = 0, 1, \dots, T-1$}
    \State $g_t = \nabla f(x_t) + \xi_t$
    \State $\tilde{g}_t = \clip(g_t;C)$
    \State $y_{t+1} \coloneqq y_t + \tilde{g}_t^{\odot 2}$
    \State $D_{t+1} \coloneqq \operatorname{diag}(\sqrt{y_{t+1}} + \epsilon_0)$
    \State $x_{t+1} \coloneqq x_t - h D_{t+1}^{-1} \tilde{g}_t$
\EndFor
\end{algorithmic}
\end{algorithm}

\subsection{Pathwise Localization}
% As in the generalized-smoothness framework, we first localize the trajectory to a level set on which the
% objective admits an effective quadratic upper model from Lemma~\ref{lem:gen-smooth}. We fix a tolerance level $F$ satisfying $F \geq f(x_0) - f^*$. We further define a predictable stopping time for the horizon $T\geq 1$ as follows
Following the generalized-smoothness framework, we first localize the trajectory to a level set where the objective admits an effective quadratic upper model (Lemma~\ref{lem:gen-smooth}). Fixing a tolerance $F \geq f(x_0) - f^*$, we define a predictable stopping time for horizon $T \geq 1$ as follows:
\[
\tau \coloneqq \min\{t \le T : f(x_t)-f^\star > F\}\wedge T.
\]
We define the localized second-moment bound $V \coloneqq (G)^2 + \sigma^2$. For $t<\tau$, Lemma~\ref{lem:levelset} gives $\|\nabla f(x_t)\|\le G$. Thus, the generalized-smoothness descent lemma (Lemma~\ref{lem:gen-smooth}) applies with local smoothness constant $L=\ell(2G)$ as long as the step remains within the local radius $G/L$. Clipping enforces $\|\tilde g_t\|\le C$, and the warm start guarantees $D_{t+1}\!\succeq\! \underline{D}I$, where $\underline D\!\coloneqq\!\sqrt{y_{\mathrm{init}}}\!+\!\epsilon_0$. Consequently, for every step taken by Algorithm \ref{alg:clipped_adagrad} before the trajectory leaves the level set, the following holds:
% We additionally define the localized second-moment bound $V \coloneqq (G)^2 + \sigma^2$. For every $t<\tau$, Lemma~\ref{lem:levelset} gives $\|\nabla f(x_t)\|\le G$. Therefore, the generalized-smoothness descent lemma (Lemma~\ref{lem:gen-smooth}) applies with local smoothness constant $L=\ell(2G)$ as long as the step remains within the local radius $G/L$. Clipping enforces the bound $\|\tilde g_t\|\le C$, and the warm start guarantees that $D_{t+1}\succeq \underline D I$, where $\underline D\coloneqq\sqrt{y_{\mathrm{init}}}+\epsilon_0$. Consequently, for every step taken by Algorithm \ref{alg:clipped_adagrad} before the trajectory leaves the level set, the following holds:

\begin{equation}
\label{eq:xx}
\|x_{t+1}-x_t\| = h\|D_{t+1}^{-1}\tilde g_t\| \le h\|D_{t+1}^{-1}\|\,\|\tilde g_t\| \le \frac{h}{\underline D}C \le \frac{G}{L}.
\end{equation}
% Since clipping enforces \(\|\tilde g_t\|\le C\) and the warm start gives
% \(D_{t+1}\succeq \underline D I\), where \(\underline D\coloneqq\sqrt{y_{\mathrm{init}}}+\epsilon_0\), it follows that, for every step taken by Algorithm~\ref{alg:clipped_adagrad}, before the trajectory leaves the level set,
% \begin{equation}
% \label{eq:xx}
% \|x_{t+1}-x_t\|
% =
% h\|D_{t+1}^{-1}\tilde g_t\|
% \le
% h\|D_{t+1}^{-1}\|\,\|\tilde g_t\|
% \le
% \frac{h}{\underline D}C
% \le
% \frac{G}{L}.
% \end{equation}
Hence, the condition for the updates of Algorithm~\ref{alg:clipped_adagrad} to stay inside the  generalized smoothness neighborhood is
\begin{equation}
\label{eq:neigh}
\boxed{\quad \frac{h}{\underline D}C \le \frac{G}{L}\quad} \quad \text{for stepsize $h >0$.}
\end{equation}

We now turn to the \emph{same-step dependence}. We begin by bounding the one-step progress using the local quadratic model. Substituting the update rule $x_{t+1} - x_t = -h D_{t+1}^{-1} \tilde{g}_t$ into the local descent lemma (Lemma~\ref{lem:gen-smooth}), we obtain
\begin{equation}
    \label{eq:descent}
    f(x_{t+1}) - f(x_t) \le -h \langle \nabla f(x_t), D_{t+1}^{-1} \tilde{g}_t \rangle + \frac{L}{2} h^2 \| D_{t+1}^{-1} \tilde{g}_t \|^2.
\end{equation}

The core difficulty lies in the inner product $-\langle \nabla f(x_t), D_{t+1}^{-1} \tilde{g}_t \rangle$, where the metric $D_{t+1}^{-1}$ is correlated with $\tilde{g}_t$. To break this dependence, we add and subtract the prior metric $D_t^{-1}$:
\[
    -\langle \nabla f(x_t), D_{t+1}^{-1} \tilde{g}_t \rangle = -\langle \nabla f(x_t), D_t^{-1} \tilde{g}_t \rangle + \textcolor{cSame}{\underbrace{\langle \nabla f(x_t), (D_t^{-1} - D_{t+1}^{-1}) \tilde{g}_t \rangle}_{\coloneqq E_t}},
\]
where \textcolor{cSame}{$E_t$} isolates the exact algebraic penalty incurred by the same-step correction. To cleanly handle the predictable inner product $\langle \nabla f(x_t), D_t^{-1} \tilde{g}_t \rangle$ above, we decompose the clipped gradient into its conditional mean and martingale noise as $\tilde{g}_t = \mu_t + \tilde{\xi}_t$, with $\mu_t \coloneqq \nabla f(x_t) + \mathbb{E}_{t-1}[r_t]$. This allows us to define the predictable preconditioned progress \textcolor{cMain}{$A_t \coloneqq \langle \nabla f(x_t), D_t^{-1} \nabla f(x_t) \rangle$} and the stopped martingale term \textcolor{cMart}{$M_t \coloneqq \langle D_t^{-1} \nabla f(x_t), \tilde{\xi}_t \rangle$}. Substituting these extracted components back into the descent inequality (\ref{eq:descent}) establishes an one-step progress bound, formalized below.

\begin{lemma}[One-step descent with same-step correction]
\label{lem:onestep}
Assume (\ref{eq:neigh}), for every \(t<\tau\),
\[
f(x_{t+1}) - f(x_t)
\le
-h\textcolor{cMain}{A_t}
-h\textcolor{cBias}{\langle D_t^{-1}\nabla f(x_t), \mathbb{E}_{t-1}[r_t]\rangle}
-h\textcolor{cMart}{M_t}
+h\textcolor{cSame}{E_t}
+\frac{L}{2}h^2\textcolor{cRem}{\|D_{t+1}^{-1}\tilde g_t\|^2}.
\]
\end{lemma}

\begin{proof}
    Proof is deferred to Appendix~\ref{sec:l3}.
\end{proof}
The point of Lemma~\ref{lem:onestep} is that the same-step issue has now been reduced to a single
deterministic term \(\textcolor{cSame}{E_t}\). We now move to bound $\textcolor{cSame}{E_t}$.

\begin{lemma}[Deterministic bound on the cumulative same-step correction]
\label{lem:samestep}
For every sample path, $\sum_{t<\tau} |\textcolor{cSame}{E_t}| \le \frac{dGC}{\underline D}.
$ Consequently,$\sum_{t<\tau} \textcolor{cSame}{E_t} \le \frac{dGC}{\underline D}.
$
\end{lemma}
\begin{proof}[Proof Sketch]
    Since $y_{t+1,i}\ge y_{t,i}$, we have $D_{t+1,i}^{-1}\le D_{t,i}^{-1}$. Hence, using
$|\nabla_i f(x_t)|\le G$ and $|\tilde g_{t,i}|\le C$, $|E_t|
\le
GC\sum_{i=1}^d \left(D_{t,i}^{-1}-D_{t+1,i}^{-1}\right).
$ Summing over $\!t<\!\tau$ telescopes coordinatewise:
$\sum_{t<\tau}\!|E_t|
\!\le\!
GC\sum_{i=1}^d \!\left(D_{0,i}^{-1}\!-\!D_{\tau,i}^{-1}\!\right)
\!\le\!
GC\sum_{i=1}^d\!D_{0,i}^{-1}
\!=\!
\frac{dGC}{\underline D}.
$ Full proof is deferred to Appendix~\ref{sec:l4}.  
\end{proof}
\begin{remark}
    Lemma~\ref{lem:samestep} is the structural reason the original same-step update remains analyzable.
Once the term \(\textcolor{cSame}{E_t}\) is isolated, the non-predictable part of the descent inequality no longer appears
as an uncontrolled stochastic contribution. It is converted into a deterministic telescoping correction
that can be bounded pathwise, enabling a high-probability
analysis of the original same-step AdaGrad.
\end{remark}
At this stage, the \emph{same-step} dependence has been fully isolated using Lemma~\ref{lem:samestep}. The remaining question is \emph{under what parameter scales} the resulting error terms of Lemma~\ref{lem:onestep} can be made compatible with a finite-horizon Euclidean target? This is where calibration enters. We choose the stepsize $(h)$, clipping threshold $(C)$, and warm start so the \textcolor{cSame}{same-step correction} and \textcolor{cBias}{clipping residual} can be absorbed into the high-probability preconditioned descent bound, while converting the \textcolor{cMain}{preconditioned estimate} to a Euclidean bound still yields a genuine finite-horizon target.
% At this stage, the \emph{same-step} dependence has been fully isolated using Lemma~\ref{lem:samestep}. Now, the remaining question is \emph{under what parameter scales}
% the resulting error terms of Lemma~\ref{lem:onestep} can be made simultaneously compatible with a finite-horizon Euclidean target?
% This is exactly where the calibration enters. We now choose the stepsize $(h)$, clipping threshold $(C)$, and warm start so that the \textcolor{cSame}{same-step correction} and \textcolor{cBias}{clipping residual} can be absorbed in the high-probability preconditioned descent bound, while the later conversion from the \textcolor{cMain}{preconditioned estimate} to a Euclidean bound still yields a genuine finite-horizon target.

\subsection{Calibrating Parameters}
\label{sec:param}
To explicitly motivate our hyperparameter choices, we first establish a preconditioned descent estimate up to the stopping time $\tau$.

\begin{lemma}[Summed descent inequality up to \(\tau\)]
\label{lem:summed_descent}
Assume the neighborhood condition (\ref{eq:neigh}) holds. Then, for every $\delta \in (0,1)$, with probability at least $1-\delta$, we have
\begin{equation}
\label{eq:3}
f(x_\tau)-f^\star + \frac{h}{2}\sum_{t<\tau} \textcolor{cMain}{A_t}
\le
\Delta_0
+ \underbrace{\textcolor{cBias}{\frac{hT V^2}{\underline D C^2}}}_{\mathclap{\substack{\text{clipping-residual}\\\text{contribution}}}}
+ \quad 
\underbrace{\textcolor{cRem}{\frac{L}{2} h^2 H_T}}_{\mathclap{\substack{\text{generalized}\\\text{smoothness remainder}}}}
+ \quad 
\underbrace{h \textcolor{cSame}{\frac{dGC}{\underline D}}}_{\mathclap{\substack{\text{same-step}\\ \text{correction}}}}
+
h\textcolor{cMart}{\Gamma\log(1/\delta)},
\end{equation}
where $H_T = d\log\left(1+\frac{TC^2}{y_{\mathrm{init}} + \varepsilon_0^2}\right) + d\frac{C^2}{y_{\mathrm{init}} + \varepsilon_0^2}$ and $\Gamma = \frac{2V+ \frac{2}{3}CG}{\underline D}$.
\end{lemma}

\begin{proof}[Proof sketch]
We sum Lemma~\ref{lem:onestep} over all \(t<\tau\). Besides the main progress term \(\sum_{t<\tau}A_t\), this produces
four error contributions: the \textcolor{cBias}{clipping-residual} term, the \textcolor{cSame}{same-step correction} term, the quadratic
\textcolor{cRem}{generalized-smoothness remainder}, and the \textcolor{cMart}{martingale fluctuation} term. The \textcolor{cBias}{clipping-residual} term is bounded by using Young's inequality, $\big(ab\leq \frac{1}{4}a^2 + b^2\big)$ with $a = \sqrt{A_t}$ and $b = V/c\sqrt{\underlineD}$, which then absorbs one
quarter of the resulting term into \(\sum_{t<\tau}A_t\), leaving a deterministic remainder of order
\textcolor{cBias}{\(TV^2/(\underline D C^2)\)}. The cumulative \textcolor{cSame}{same-step term} is controlled pathwise by Lemma~\ref{lem:samestep}, giving a
deterministic contribution of order \textcolor{cSame}{\(dGC/\underline D\)}. The quadratic \textcolor{cRem}{generalized-smoothness
remainder} is bounded by the pathwise estimate on \(\sum_{t<\tau}\|D_{t+1}^{-1}\tilde g_t\|^2\), producing
the term \textcolor{cRem}{\(\frac{L}{2}h^2H_T\)}. Finally, the \textcolor{cMart}{martingale term} is controlled later with high probability by
invoking a Bernstein--Freedman supermartingale bound with probability atleast $1-\delta$, which turns it into a contribution of order
\textcolor{cMart}{\(\frac14\sum_{t<\tau}A_t+\frac{2V + \frac{2}{3}CG}{\underline D}\log(1/\delta)\)}. Substituting these bounds into the summed form
of Lemma~\ref{lem:onestep} and rearranging gives the stated inequality in (\ref{eq:3}). Full proof is deferred to Appendix~\ref{sec:l5}.
\end{proof}

After this preconditioned estimate is obtained, converting it into
a Euclidean stationarity bound introduces the high-probability upper bound on $D^{-1}_t$.

\begin{lemma}[High-probability bound on $D_t$]
\label{lem:HPD}
Define the stopped clipped-gradient energy $
\widetilde Y_T\coloneqq\sum_{t=0}^{T-1}\|\tilde g_t\|^2\one_{\{t<\tau\}}.$ Then for every $\delta\in(0,1)$, with probability at least $1-\delta$,
\begin{small}
\[
\widetilde Y_T
\le
B_{T,\delta}\coloneqq
TV+\sqrt{2TC^2V\log\frac{1}{\delta}}+\frac{2C^2}{3}\log\frac{1}{\delta}.
\]
\end{small}
Consequently, on $\{\tau=T\}\cap\{\widetilde Y_T\le B_{T,\delta}\}$, we have, $\max_{t\le T}\|D_t\|
\le
\overlineD_{\delta}\coloneqq\sqrt{\yinit+B_{T,\delta}}+\epsz.
$
\end{lemma}

\begin{proof}
    Proof is deferred to Appendix~\ref{sec:hpbbb}.
\end{proof}
\begin{remark}
    From Lemma~\ref{lem:HPD}, \(D_t \preceq \overline D_\delta I\) for all \(t\le T\), so \(D_t^{-1} \succeq \overline D_\delta^{-1} I\). Hence \(A_t=\langle \nabla f(x_t),D_t^{-1}\nabla f(x_t)\rangle \ge \overline D_\delta^{-1}\|\nabla f(x_t)\|^2\), and therefore \(\sum_{t<\tau}\|\nabla f(x_t)\|^2 \le \overline D_\delta \sum_{t<\tau} A_t\). Combining this with the preconditioned bound from Lemma~\ref{lem:summed_descent} and dividing by \(T\) yields the Euclidean conversion factor \(\overline D_\delta/(hT)\).
\end{remark}

Lemma~\ref{lem:summed_descent} shows that, before we can prove a Euclidean guarantee, we must control all
four error terms. Among all the terms in the summed descent bound, the \textcolor{cBias}{clipping-residual contribution}, the \textcolor{cSame}{same-step correction}, and the later preconditioned-to-Euclidean conversion (Lemma~\ref{lem:HPD}) are the ones that determine the joint choice of \(h\), \(C\), and the warm-start floor \(\underline D\). Indeed, before coupling \(\underline D\) to \(C\), these two deterministic terms pull \(C\) in opposite directions: \textcolor{cBias}{\(\frac{hTV^2}{\underline D C^2}\)} decreases as \(C\) grows, whereas \textcolor{cSame}{\(\frac{hdGC}{\underline D}\)} increases with \(C\). \textbf{This is the core calibration conflict}. If \(\underline D\) were treated as independent of \(C\), then balancing these two terms with \(h\) essentially fixed would force \(C^3 \asymp T\), and their common scale would still grow like \(\frac{hT^{1/3}}{\underline D}\), hence unbounded. The purpose of the matched warm start is therefore to remove this harmful \(C\)-dependence from the \textcolor{cSame}{same-step term}. By contrast, the \textcolor{cRem}{quadratic generalized-smoothness remainder} and the \textcolor{cMart}{martingale fluctuation} are still
present in the final bound, but they do not drive the main \(h\)–\(C\)–\(\underline D\) coupling. 

This is why the calibration is chosen to keep the following three
quantities at the right scale:
\begin{small}
\[
\underbrace{\textcolor{cBias}{\frac{hT V^2}{\underline D C^2}}}_{\text{clipping-residual contribution}},
\qquad
\underbrace{\textcolor{cSame}{\frac{hdGC}{\underline D}}}_{\text{same-step correction}},
\qquad
\underbrace{\frac{\overline D_\delta}{hT}}_{\text{preconditioned-to-Euclidean conversion}}.
\]
\end{small}

First, we tie the warm-start floor to the clipping scale by setting $y_{\mathrm{init}} = C^2,$ $\epsilon_0 = C.$ Then $\underline D = \sqrt{y_{\mathrm{init}}}+\epsilon_0 = 2C,$
so the \textcolor{cSame}{same-step correction} penalty collapses from \textcolor{cSame}{\(\frac{hdGC}{\underline D}\)} to $\textcolor{cSame}{\frac{h\,dG}{2}}$, leaving \(C\) free to control only the \textcolor{cBias}{clipping-residual contribution}.
% which is now independent of the clipping threshold. This is the key matched cancellation.

Second, once \(\underline D=2C\), the \textcolor{cBias}{clipping-residual} contribution becomes \(\textcolor{cBias}{hTV^2/C^3}\). Since the feasibility condition is imposed at the fixed anchored level \(F\), the \textcolor{cBias}{clipping-residual term \(\frac{hTV^2}{2C^3}\)} must remain \(\mathcal{O}(1)\). This requires \(\frac{hT}{C^3}=\mathcal{O}(1)\), which leads to the natural cubic scale \(C^3 \asymp hT\). The remaining term, \(\overline D_\delta/(hT)\), then measures the cost of converting the resulting \textcolor{cMain}{preconditioned estimate} into a Euclidean bound. To obtain a true finite-horizon Euclidean target, we keep \(h\) almost constant up to
polylogarithmic factors and choose \(C\) on the cubic scale by setting, for parameters
\(\eta,\kappa>0\):
\begin{equation}
\label{eq:calib}
\boxed{h = \frac{\eta}{\sqrt{\Lambda_{T,\delta}}},
\qquad
C = \kappa T^{1/3},
\qquad
\Lambda_{T,\delta} \coloneqq 1+\log(1+T)+\log^2\!\frac{2}{\delta}.}
\end{equation}
This is the logic behind the calibration. It is not an ad hoc theorem-only schedule. It is the scale at
which the \textcolor{cSame}{same-step correction} stops growing with \(C\), the \textcolor{cBias}{clipping-residual} contribution stops growing
with \(T\), and the Euclidean conversion still yields an \(\widetilde{\mathcal O}(T^{-1/2})\) bound. 

\emph{We can now state the main finite-horizon theorem.}

\begin{theorem}[True finite-horizon Euclidean stationarity for clipped same-step AdaGrad]
\label{thm:main}
Suppose Assumptions~\ref{ass:1},~\ref{ass:gen}, and~\ref{ass:noise} holds. Let
\[
\Delta_0 \coloneqq f(x_0)-f^\star > 0,
\quad
F^\circ \coloneqq 2\Delta_0,
\quad
G^\circ \coloneqq G(F^\circ),
\quad
L^\circ \coloneqq \ell(2G^\circ),
\quad
V^\circ \coloneqq (G^\circ)^2+\sigma^2.
\]
Fix \(\kappa>0\), choose \((h,C)\) from~\eqref{eq:calib}, and set $y_{\mathrm{init}} \coloneqq C^2,$ $
\epsilon_0 \coloneqq C$. Suppose moreover that
\[
\eta \le
\min\left\{
\frac{2G^\circ}{L^\circ},
\sqrt{\frac{\Delta_0}{4L^\circ d}},
\frac{\Delta_0\kappa}{8V^\circ},
\frac{\Delta_0\kappa^3}{4(V^\circ)^2},
\frac{3\Delta_0}{8G^\circ},
\frac{\Delta_0}{dG^\circ}
\right\},
\] then the neighborhood condition (\ref{eq:neigh}) holds at $F^\circ$. Then, with probability at least \(1-\delta\),
\[
f(x_t)-f^\star \le 2\Delta_0
\qquad \text{for all } t=0,\dots,T, \quad \text{and}
\]
\[
    \frac{1}{T}\!\sum_{t=0}^{T-1}\!\|\nabla f(x_t)\|^2\!
    \le\!
    \frac{4\Delta_0\sqrt{\Lambda_{T,\delta}}}{\eta\sqrt{T}}
    \!\!\left[\!
        \!\left(\!
        V^\circ
        \!+\!
        \kappa\sqrt{2V^\circ\log\frac{2}{\delta}}\,T^{-1/6}
        \!+\!
        \kappa^2\!\!\left(\!1\!+\!\frac{2}{3}\log\frac{2}{\delta}\!\right)\!T^{-1/3}\!
        \right)^{\!\!1/2}
        \!\!\!\!\!\!\!+
        \!\kappa T^{-1/6}
    \!\right]\!\!.
\]
Consequently, for fixed problem parameters, $
\frac{1}{T}\sum_{t=0}^{T-1}\|\nabla f(x_t)\|^2 \le \epsilon
\quad\Longrightarrow\quad
T = \widetilde{\mathcal O}(\epsilon^{-2}).$
\end{theorem}
    
\begin{proof}
    Proof is deferred to Appendix~\ref{sec:theorem1proof}.
\end{proof}
The choice to anchor Theorem 1 at $F^\circ = 2\Delta_0$ is an analytic convenience rather than a strict algorithmic limitation. Setting the target to twice the initial suboptimality constructs a symmetric feasibility budget: exactly $\Delta_0$ accounts for the initial gap, leaving an equal allowance of $\Delta_0$ to absorb the cumulative algorithmic and stochastic penalties from Lemma~\ref{lem:summed_descent} (specifically: the \textcolor{cBias}{clipping-residual}, \textcolor{cSame}{same-step correction}, \textcolor{cMart}{martingale fluctuation}, and \textcolor{cRem}{smoothness remainder}). Importantly, the $\tilde{\mathcal{O}}(T^{-1/2})$ convergence guarantee holds for any arbitrary target level $F > \Delta_0$. For any strictly positive slack $F - \Delta_0 > 0$, one can always scale down the step-size parameter $\eta$ sufficiently to compress the error budget within this margin. We refer the reader to Appendix~\ref{sec:appex} for a longer discussion.
\begin{remark}
The parameter choices in Theorem~\ref{thm:main} serve as an explicit \emph{proof calibration}, not a practical tuning prescription. Their purpose is to exhibit one concrete regime where the clipping-residual term, the same-step correction, and the final preconditioned-to-Euclidean conversion are simultaneously controlled. This calibration acts as a constructive witness, proving the original same-step AdaGrad update can be stabilized under generalized smoothness and heavy-tailed noise. Standard in non-asymptotic optimization theory, these schedules expose the sharp structural balances required by the analysis, rather than mimicking default practitioner heuristics.
% The parameter choices in Theorem~\ref{thm:main} should be interpreted as an explicit \emph{proof calibration}, not as a practical tuning prescription. Their purpose is to exhibit one concrete regime in which the clipping-residual term, the same-step correction, and the final preconditioned-to-Euclidean conversion are simultaneously controlled at the correct scale. In this sense, the calibration acts as a constructive witness, proving that the original same-step AdaGrad update can be stabilized under generalized smoothness and heavy-tailed noise. As is standard in non-asymptotic optimization theory, these schedules are chosen to expose the sharp structural balances required by the analysis, rather than to mimic default practitioner heuristics.
% The calibration in Theorem~\ref{thm:main} is introduced as a constructive proof device, not as a claim about standard practical tuning. It identifies one explicit parameter regime in which the original same-step AdaGrad update provably avoids the obstruction of Proposition~\ref{thm:anisotropic_miscalibration} and admits a finite-horizon high-probability Euclidean guarantee. In this sense, the calibration is strictly meant to establish the structural feasibility of the theorem rather than the empirical optimality of constants. Because these constraints exist to neutralize worst-case stochastic penalties, evaluating this specific theoretical calibration in standard empirical benchmarks would yield no mathematically relevant insights regarding the theorem's validity.
\end{remark}

\section{Discussion of the Results and Conclusion}
Theorem~\ref{thm:main} establishes a finite-horizon Euclidean stationarity bound for the original clipped same-step AdaGrad update under generalized smoothness and heavy-tailed noise, with polylogarithmic dependence on the confidence level $\delta$. Crucially, our $\tilde{\mathcal{O}}(\epsilon^{-2})$ complexity matches the $\Omega(\epsilon^{-2})$ information-theoretic lower bound for standard non-convex stochastic optimization \citep{arjevani2023lower} up to logarithmic factors. While our high-probability, heavy-tail setting aligns with recent clipped stochastic analyses \citep{cutkosky2021high,pmlr-v202-sadiev23a,nguyen2023improved,chezhegov2025clipping}, we distinctively resolve this for the original coordinate-wise same-step update.
% Theorem~\ref{thm:main} provides a true finite-horizon Euclidean stationarity bound for the original clipped same-step AdaGrad update under generalized smoothness and heavy-tailed noise, with polylogarithmic dependence on the confidence level $\delta$. Crucially, our $\tilde{\mathcal{O}}(\epsilon^{-2})$ complexity matches the fundamental information-theoretic lower bound of $\Omega(\epsilon^{-2})$ established by~\citep{arjevani2023lower} for standard non-convex stochastic optimization upto logarithmic factor. At the same time, our result is obtained in a clipped, high-probability, heavy-tail regime closer in spirit to recent clipped stochastic analyses \citep{cutkosky2021high,pmlr-v202-sadiev23a,nguyen2023improved,chezhegov2025clipping}, but for the original same-step coordinate-wise AdaGrad update under generalized smoothness.

It is also important to distinguish the present theorem from recent analyses of adaptive methods under non-uniform or generalized smoothness. These works show that convergence can still be obtained when local curvature grows with the gradient norm, for settings ranging from adaptive SGD to AdaGrad and Adam (and its variants), but they do not provide the present finite-horizon high-probability Euclidean guarantee for the original same-step coordinate-wise AdaGrad update under heavy-tailed noise \citep{faw2023beyond,wang2023convergence,li2023convergence,wang2024provable}. Further, up to the differences in logarithmic factors, these complexities coincide with the best known ones for Clipped-SGD~\citep{pmlr-v202-sadiev23a,nguyen2023improved}\footnote{Their analysis assumes $L$ smoothness.}. 

Compared with existing high-probability AdaGrad analyses, our result applies in a more demanding regime. Prior coordinate-wise AdaGrad bounds obtain sharper dimension dependence, but assume classical \(L\)-smoothness together with almost-sure bounded gradients and sub-Gaussian noise \citep{zhou2024on}; other results study either delayed AdaGrad under sub-Gaussian noise \citep{li2020high} or scalar AdaGrad-Norm rather than the original coordinate-wise update \citep{kavis2022high}. Heavy-tailed clipped AdaGrad results are closer, but existing theorems either analyze scalar/norm-based AdaGrad with an additional bounded-risk assumption \(f(x)\le M\) \citep{li2023high}, or treat delayed coordinate-wise variants rather than the original same-step update \citep{chezhegov2025clipping}.

% On the other hand, recent clipping-based high-probability results for adaptive methods under heavy-tailed noise focus on delayed, scalar/norm-based, or globally smooth variants such as delayed AdaGrad with momentum, clipped adaptive methods under additional boundedness assumptions, or AdaGrad-Norm/Adam-Norm under classical smoothness \citep{li2023high,chezhegov2025clipping}. Unlike prior works, we directly analyze the original same-step anisotropic preconditioner under generalized smoothness, without relying on delays or scalar surrogates.

It is also worth mentioning that the existing high-probability
complexities for Adam/AdaGrad-type (without clipping)
methods either have inverse power dependence on $\delta$~
\cite{wang2023convergence} or have polylogarithmic dependence
on $\delta$ but rely on the assumption that the noise is sub-Gaussian/bounded~\cite{li2020high,liu2023high,li2023convergence}, which is stronger than bounded variance assumption.

\clearpage
\bibliographystyle{plainnat} % NeurIPS standard bibliography style
\bibliography{reference.bib} % The name of your .bib file without the .bib extension
\newpage
\appendix
\renewcommand{\theequation}{A.\arabic{equation}}
\setcounter{equation}{0}
\section{Proofs}

\subsection{Anisotropic Miscalibration and Failure of AdaGrad}
\label{sec:ams}
\begin{proposition}[Anisotropic miscalibration lower bound for AdaGrad]
\label{thm:anisotropic_miscalibration_}
Let
\[
    f(x_1,x_2)=\frac14 x_1^4+\frac12 x_1^2+8x_2^2
\]
and \(x_0=(2,1)\). Run \textbf{Algorithm}~\ref{alg:non_lagged_adagrad}
with \(d=2\), \(h=1\), \(y_{\mathrm{init}}=0\), and \(\epsilon_0=1\).
For every horizon \(T\ge 2\) and every noise scale \(\sigma>0\), there exists
a stochastic oracle satisfying Assumption~\ref{ass:noise} with noise level
\(\sigma\) such that, with probability at least
\begin{small}
\[
    p_T\coloneqq\frac{\sigma^2}
    {\bigl(32(T-1)-10\bigr)^2+\sigma^2},
\]
\end{small}
the following inequalities hold simultaneously for every \(t=1,\dots,T-1\):
\[
    \partial_{11}^2 f(x_t)<\partial_{22}^2 f(x_t),\qquad
    (D_t)_{11}>\frac{16}{9}(T-1)(D_t)_{22},\qquad
    \|\nabla f(x_t)\|>1.
\]
Consequently, since \(\|\nabla f(x_0)\|>1\), we have 
\[\mathbb P\!\left(
    \frac1T\sum_{t=0}^{T-1}\|\nabla f(x_t)\|^2>1
    \right)\ge p_T\]. 
Therefore, any uniform guarantee of the form 
\[\mathbb P\!\left(
    \frac1T\sum_{t=0}^{T-1}\|\nabla f(x_t)\|^2\le 1
    \right)\ge 1-\delta\]
    over all stochastic oracles satisfying Assumption~\ref{ass:noise} requires 
    \[T\ge
    1+\frac{1}{32}
    \left(
        10+\sigma\sqrt{\frac{1-\delta}{\delta}}
    \right)\].
    
\end{proposition}
\begin{proof}
Set
\[
B\coloneqq32(T-1)-10,
\qquad
p_T\coloneqq\frac{\sigma^2}{B^2+\sigma^2}.
\]
We define the stochastic oracle as follows. At time \(t=0\),
\[
g_0=\nabla f(x_0)+\xi_0,
\qquad
\xi_0=(\zeta,0),
\]
where
\[
\zeta=
\begin{cases}
B, & \text{with probability } p_T,\\[1ex]
-\dfrac{\sigma^2}{B}, & \text{with probability } 1-p_T.
\end{cases}
\]
For every \(t\ge 1\), we set
\[
g_t=\nabla f(x_t),
\qquad\text{i.e.}\qquad
\xi_t\equiv 0.
\]

We first verify Assumption~\ref{ass:noise}. Since
\[
p_T B+\Bigl(1-p_T\Bigr)\Bigl(-\frac{\sigma^2}{B}\Bigr)=0,
\]
we have \(\mathbb E[\zeta]=0\), hence \(\mathbb E[\xi_0]=0\). Also,
\[
\mathbb E[\zeta^2]
=
p_T B^2+\Bigl(1-p_T\Bigr)\frac{\sigma^4}{B^2}
=
\frac{\sigma^2 B^2}{B^2+\sigma^2}
+
\frac{B^2}{B^2+\sigma^2}\frac{\sigma^4}{B^2}
=
\sigma^2.
\]
Therefore
\[
\mathbb E\|\xi_0\|^2=\sigma^2,
\qquad
\mathbb E\|\xi_t\|^2=0\le \sigma^2 \quad (t\ge 1),
\]
and the oracle satisfies Assumption~\ref{ass:noise}.

Now define the event
\[
\mathcal E\coloneqq\{\zeta=B\}.
\]
By construction,
\[
\mathbb P(\mathcal E)=p_T.
\]
We will show that on \(\mathcal E\), for every \(t=1,\dots,T-1\),
\[
\partial_{11}^2 f(x_t)<\partial_{22}^2 f(x_t),\qquad
(D_t)_{11}>\frac{16}{9}(T-1)(D_t)_{22},\qquad
\|\nabla f(x_t)\|>1.
\]

\medskip
\noindent
\textbf{Step 1: explicit first iterate and first denominator on \(\mathcal E\).}

For
\[
f(x_1,x_2)=\frac14 x_1^4+\frac12 x_1^2+8x_2^2,
\]
we have
\[
\nabla f(x_1,x_2)=
\begin{pmatrix}
x_1^3+x_1\\
16x_2
\end{pmatrix},
\qquad
\nabla^2 f(x_1,x_2)=
\begin{pmatrix}
3x_1^2+1 & 0\\
0 & 16
\end{pmatrix}.
\]
At the starting point \(x_0=(2,1)\),
\[
\nabla f(x_0)=(10,16).
\]
Hence, on \(\mathcal E\),
\[
g_0=(10+B,16)=\bigl(32(T-1),16\bigr).
\]
Since \(y_{\mathrm{init}}=0\) and \(\epsilon_0=1\), the non-lagged AdaGrad update gives
\[
y_1=g_0^{\odot 2}
=
\bigl((32(T-1))^2,16^2\bigr),
\]
and therefore
\[
D_1
=
\diag\bigl(32(T-1)+1,17\bigr).
\]
Because \(h=1\),
\[
x_1
=
x_0-D_1^{-1}g_0
=
\left(
2-\frac{32(T-1)}{32(T-1)+1},
\,
1-\frac{16}{17}
\right)
=
\left(
1+\frac{1}{32(T-1)+1},
\,
\frac1{17}
\right).
\]
In particular,
\[
1<x_{1,1}<\frac43,
\qquad
x_{1,2}=\frac1{17}>0.
\]

\medskip
\noindent
\textbf{Step 2: control of the second coordinate.}

For every \(t\ge 1\), the oracle is noiseless, so
\[
g_{t,2}=16x_{t,2},
\qquad
y_{t+1,2}=y_{t,2}+g_{t,2}^2.
\]
We claim that for every \(t\ge 1\),
\begin{equation}
\label{eq:x2_decay_clean}
0<x_{t,2}\le \frac{1}{17\cdot 9^{\,t-1}},
\end{equation}
and
\begin{equation}
\label{eq:y2_bound_clean}
y_{t,2}<257.
\end{equation}

We prove both statements simultaneously by induction on \(t\).

For \(t=1\), we already computed
\[
x_{1,2}=\frac{1}{17},
\qquad
y_{1,2}=16^2=256<257,
\]
so \eqref{eq:x2_decay_clean}--\eqref{eq:y2_bound_clean} hold.

Now assume that for every \(s=1,\dots,t\),
\[
0<x_{s,2}\le \frac{1}{17\cdot 9^{\,s-1}},
\qquad
y_{s,2}<257.
\]
We first bound \(y_{t+1,2}\). Since \(y_{1,2}=256\) and \(g_{s,2}=16x_{s,2}\) for \(s\ge 1\),
\[
y_{t+1,2}
=
256+\sum_{s=1}^{t} g_{s,2}^2
=
256+\sum_{s=1}^{t} (16x_{s,2})^2.
\]
Using the induction hypothesis,
\[
(16x_{s,2})^2
\le
\frac{256}{17^2\,9^{\,2(s-1)}}
=
\frac{256}{289}\cdot \frac{1}{81^{\,s-1}}.
\]
Therefore
\[
y_{t+1,2}
\le
256+\frac{256}{289}\sum_{j=0}^{t-1}\frac{1}{81^j}
<
256+\frac{256}{289}\cdot \frac{1}{1-1/81}
=
256+\frac{256}{289}\cdot \frac{81}{80}
<
257.
\]
Thus \eqref{eq:y2_bound_clean} holds at time \(t+1\).

Next, since \(y_{t+1,2}\ge y_{1,2}=256\), we have
\[
(D_{t+1})_{22}=\sqrt{y_{t+1,2}}+1\ge 17,
\]
and since \(y_{t+1,2}<257\), we also have
\[
(D_{t+1})_{22}<\sqrt{257}+1<18.
\]
Hence
\[
0<1-\frac{16}{(D_{t+1})_{22}}<1-\frac{16}{18}=\frac19.
\]
Therefore
\[
x_{t+1,2}
=
x_{t,2}-\frac{16x_{t,2}}{(D_{t+1})_{22}}
=
x_{t,2}\left(1-\frac{16}{(D_{t+1})_{22}}\right),
\]
so in particular \(x_{t+1,2}>0\), and moreover
\[
x_{t+1,2}\le \frac{x_{t,2}}{9}
\le
\frac{1}{17\cdot 9^{\,t}}.
\]
Thus \eqref{eq:x2_decay_clean} holds at time \(t+1\), completing the induction.

Consequently, for every \(t\ge 1\),
\[
(D_t)_{22}=\sqrt{y_{t,2}}+1<\sqrt{257}+1<18.
\]

\medskip
\noindent
\textbf{Step 3: control of the first coordinate.}

For every \(t\ge 1\), the oracle is noiseless, so
\[
g_{t,1}=x_{t,1}^3+x_{t,1}.
\]
Since \(y_{t,1}\) is nondecreasing and \(y_{1,1}=(32(T-1))^2\),
\[
(D_t)_{11}=\sqrt{y_{t,1}}+1\ge 32(T-1)+1
\qquad\text{for all } t\ge 1.
\]
We first show that \(x_{t,1}\) stays below \(4/3\). Since \(x_{1,1}<4/3\), it suffices to show that
\(x_{t,1}\) is nonincreasing. But if \(x_{t,1}>0\), then \(g_{t,1}=x_{t,1}^3+x_{t,1}>0\), so
\[
x_{t+1,1}=x_{t,1}-\frac{g_{t,1}}{(D_{t+1})_{11}}<x_{t,1}.
\]
Thus, by induction,
\[
0<x_{t,1}\le x_{1,1}<\frac43
\qquad\text{for all } t\ge 1.
\]
Using \(x_{t,1}<4/3\), we obtain
\[
g_{t,1}=x_{t,1}^3+x_{t,1}
\le
\left(\frac43\right)^3+\frac43
=
\frac{100}{27}
<4.
\]
Hence, for every \(t\ge 1\),
\[
x_{t+1,1}
=
x_{t,1}-\frac{g_{t,1}}{(D_{t+1})_{11}}
\ge
x_{t,1}-\frac{4}{32(T-1)+1}.
\]
Iterating this bound from time \(1\) up to time \(t\le T-1\) yields
\[
x_{t,1}
\ge
x_{1,1}-\frac{4(t-1)}{32(T-1)+1}
\ge
x_{1,1}-\frac{4(T-2)}{32(T-1)+1}.
\]
Using the explicit formula for \(x_{1,1}\),
\[
x_{t,1}
\ge
1+\frac{1}{32(T-1)+1}-\frac{4(T-2)}{32(T-1)+1}
=
\frac{28T-22}{32T-31}.
\]
Since
\[
8(28T-22)=224T-176>224T-217=7(32T-31),
\]
we get
\[
x_{t,1}>\frac78
\qquad\text{for all } t=1,\dots,T-1.
\]

\medskip
\noindent
\textbf{Step 4: anisotropic miscalibration and nonstationarity on \(\mathcal E\).}

Fix \(t\in\{1,\dots,T-1\}\). Since \(x_{t,1}<4/3\),
\[
\partial_{11}^2 f(x_t)=3x_{t,1}^2+1
<
3\left(\frac43\right)^2+1
=
\frac{19}{3}
<16
=
\partial_{22}^2 f(x_t).
\]
This proves
\[
\partial_{11}^2 f(x_t)<\partial_{22}^2 f(x_t).
\]

Next, by Steps 2 and 3,
\[
(D_t)_{11}\ge 32(T-1)+1,
\qquad
(D_t)_{22}<18.
\]
Therefore
\[
(D_t)_{11}
>
32(T-1)
=
\frac{16}{9}(T-1)\cdot 18
>
\frac{16}{9}(T-1)(D_t)_{22}.
\]
This proves the claimed anisotropic denominator inflation.

Finally, since \(x_{t,1}>7/8\),
\[
\partial_1 f(x_t)=x_{t,1}^3+x_{t,1}
>
\left(\frac78\right)^3+\frac78
=
\frac{343}{512}+\frac{448}{512}
=
\frac{791}{512}
>1.
\]
Hence
\[
\|\nabla f(x_t)\|\ge |\partial_1 f(x_t)|>1.
\]
Thus all three displayed properties in the theorem hold for every \(t=1,\dots,T-1\) on the event
\(\mathcal E\).

\medskip
\noindent
\textbf{Step 5: lower bound on the averaged Euclidean stationarity measure.}

On \(\mathcal E\), we have just shown that
\[
\|\nabla f(x_t)\|>1
\qquad\text{for all } t=1,\dots,T-1.
\]
Also
\[
\|\nabla f(x_0)\|^2=10^2+16^2=356>1.
\]
Therefore, on \(\mathcal E\),
\[
\frac1T\sum_{t=0}^{T-1}\|\nabla f(x_t)\|^2>1.
\]
Since \(\mathbb P(\mathcal E)=p_T\), it follows that
\[
\mathbb P\!\left(
\frac1T\sum_{t=0}^{T-1}\|\nabla f(x_t)\|^2>1
\right)\ge p_T.
\]

\medskip
\noindent
\textbf{Step 6: consequence for any \((1-\delta)\)-high-probability guarantee.}

Suppose
\[
\mathbb P\!\left(
\frac1T\sum_{t=0}^{T-1}\|\nabla f(x_t)\|^2\le 1
\right)\ge 1-\delta.
\]
Then necessarily
\[
\mathbb P\!\left(
\frac1T\sum_{t=0}^{T-1}\|\nabla f(x_t)\|^2>1
\right)\le \delta.
\]
By the lower bound proved above, this implies
\[
p_T\le \delta.
\]
Using
\[
p_T=\frac{\sigma^2}{B^2+\sigma^2},
\qquad
B=32(T-1)-10,
\]
we get
\[
\frac{\sigma^2}{B^2+\sigma^2}\le \delta
\quad\Longrightarrow\quad
\sigma^2\le \delta(B^2+\sigma^2)
\quad\Longrightarrow\quad
(1-\delta)\sigma^2\le \delta B^2.
\]
Hence
\[
B\ge \sigma\sqrt{\frac{1-\delta}{\delta}}.
\]
Substituting back \(B=32(T-1)-10\) yields
\[
32(T-1)-10\ge \sigma\sqrt{\frac{1-\delta}{\delta}},
\]
that is,
\[
T\ge 1+\frac{1}{32}\left(10+\sigma\sqrt{\frac{1-\delta}{\delta}}\right).
\]
This proves the theorem.
\end{proof}

Before proving Theorem~\ref{thm:main}, we introduce several auxiliary lemmas and definitions that will be used throughout the proof.

\subsection{Proofs of Auxiliary Lemmas}

\begin{definition}[Predictable level-set stopping time]
\label{def:tau}
Define
\[
\tau := \min\{t\le T:\ f(x_t)-f^\star>F\}\wedge T.
\]
Then $\one_{\{t<\tau\}}$ is $\calF_{t-1}$-measurable.
\end{definition}

\begin{lemma}[Before $\tau$: gradient bound]
\label{lem:beforetau_grad}
For all $t<\tau$, $f(x_t)-f^\star\le F$ and thus $\|\nabla f(x_t)\|\le G$.
\end{lemma}

\begin{proof}
Immediate from Definition~\ref{def:tau} and Lemma~\ref{lem:levelset}.
\end{proof}

\subsubsection{Proof of Lemma~\ref{lem:onestep}}
\label{sec:l3}
\begin{lemma}[One-step descent with same-step correction]
\label{lem:onestep_}
Assume (\ref{eq:neigh}), for every \(t<\tau\),
\begin{equation}
\label{eq:lemma3_eq}
f(x_{t+1}) - f(x_t)
\le
-h\textcolor{cMain}{A_t}
-h\textcolor{cBias}{\langle D_t^{-1}\nabla f(x_t), \mathbb{E}_{t-1}[r_t]\rangle}
-h\textcolor{cMart}{M_t}
+h\textcolor{cSame}{E_t}
+\frac{L}{2}h^2\textcolor{cRem}{\|D_{t+1}^{-1}\tilde g_t\|^2}.
\end{equation}
\end{lemma}

\begin{proof}
For $t<\tau$, Lemma~\ref{lem:beforetau_grad} and (\ref{eq:xx}) allow us to apply Lemma~\ref{lem:gen-smooth} at $(x_t,x_{t+1})$.
Substituting $x_{t+1}-x_t=-hD_{t+1}^{-1}\tilde g_t$ from Algorithm~\ref{alg:clipped_adagrad} gives
\[
f(x_{t+1})-f(x_t)
\le
-h\inner{\nabla f(x_t)}{D_{t+1}^{-1}\tilde g_t}
+\frac{L}{2}h^2\|D_{t+1}^{-1}\tilde g_t\|^2.
\]
Add and subtract $D_t^{-1}$ in the first-order term:
\[
-\inner{\nabla f(x_t)}{D_{t+1}^{-1}\tilde g_t}
=
-\inner{\nabla f(x_t)}{D_t^{-1}\tilde g_t}+E_t,
\]
where $E_t = \langle \nabla f(x_t), (D_t^{-1} - D_{t+1}^{-1}) \tilde{g}_t \rangle$. Now decompose $\tilde g_t=\mu_t+\tilde\xi_t$ and recall $\mu_t=\nabla f(x_t)+\EE_{t-1}[r_t]$ (from \textbf{Clippied Noise} paragraph of Section~\ref{sec:upper}).
This yields
\[
-\inner{\nabla f(x_t)}{D_t^{-1}\tilde g_t}
=
-A_t
-\inner{D_t^{-1}\nabla f(x_t)}{\EE_{t-1}[r_t]}
-M_t,
\]
which proves the claim.
\end{proof}

Before presenting the proof of Lemma~\ref{lem:summed_descent}, we must sum (\ref{eq:lemma3_eq}) in Lemma~\ref{lem:onestep_} and individually bound each resulting term.

\subsubsection{Bounding $\textcolor{cBias}{\inner{D_t^{-1}\nabla f(x_t)}{\EE_{t-1}[r_t]}}$}

\begin{lemma}[Clipping residual first moment]
\label{lem:clip1}
For all $t<\tau$,
\[
\EE_{t-1}\|r_t\|\le \frac{\EE_{t-1}\|g_t\|^2}{C}\le \frac{V}{C}.
\]
\end{lemma}

\begin{proof}
$\|r_t\|=(\|g_t\|-C)_+\le \|g_t\|^2/C$ implies $\EE_{t-1}\|r_t\|\le \EE_{t-1}\|g_t\|^2/C$.
For $t<\tau$, $\|\nabla f(x_t)\|\le G$ and $\EE_{t-1}\|\xi_t\|^2\le \sigma^2$ give
$\EE_{t-1}\|g_t\|^2=\|\nabla f(x_t)\|^2+\EE_{t-1}\|\xi_t\|^2\le G^2+\sigma^2=V$.
\end{proof}

\begin{lemma}[Bias coupled to $A_t$]
\label{lem:bias-coupled}
For all $t<\tau$,
\[
\textcolor{cBias}{\Big|\inner{D_t^{-1}\nabla f(x_t)}{\EE_{t-1}[r_t]}\Big|}
\le
\frac{1}{4}A_t + \frac{V^2}{\underlineD\,C^2}.
\]
\end{lemma}

\begin{proof}
$\|D_t^{-1}\nabla f(x_t)\|^2=\inner{\nabla f(x_t)}{D_t^{-2}\nabla f(x_t)}
\le \|D_t^{-1}\|\,A_t\le \frac{1}{\underlineD}A_t$.
Then
\[
\Big|\inner{D_t^{-1}\nabla f(x_t)}{\EE_{t-1}[r_t]}\Big|
\le
\|D_t^{-1}\nabla f(x_t)\|\,\|\EE_{t-1}[r_t]\|
\le
\sqrt{\frac{A_t}{\underlineD}}\;\EE_{t-1}\|r_t\|
\le
\sqrt{\frac{A_t}{\underlineD}}\;\frac{V}{C},
\]
and apply $ab\le \frac14 a^2+b^2$ with $a=\sqrt{A_t}$ and $b=\frac{V}{C\sqrt{\underlineD}}$.
\end{proof}

\subsubsection{Bounding \textcolor{cSame}
{$\sum_{t<\tau}E_t$}}
\label{sec:l4}
\begin{lemma}[Deterministic bound on the cumulative same-step correction]
\label{lem:samestep-nl}
For every sample path,
\[
\sum_{t<\tau} |E_t|
\le
\frac{dGC}{\underlineD}.
\]
Consequently,
\[
\sum_{t<\tau} E_t
\le
\frac{dGC}{\underlineD}.
\]
\end{lemma}

\begin{proof}
Fix a coordinate $i$.
Because $y_{t+1,i}\ge y_{t,i}$, the scalar sequence $(D_{t,i}^{-1})_{t\ge 0}$ is nonincreasing, so
\[
D_{t,i}^{-1}-D_{t+1,i}^{-1}\ge 0.
\]
For $t<\tau$, Lemma~\ref{lem:beforetau_grad} gives $|\nabla_i f(x_t)|\le G$, while clipping gives $|\tilde g_{t,i}|\le C$.
Hence
\[
|E_t|
\le
\sum_{i=1}^d |\nabla_i f(x_t)|\,(D_{t,i}^{-1}-D_{t+1,i}^{-1})\,|\tilde g_{t,i}|
\le
GC\sum_{i=1}^d (D_{t,i}^{-1}-D_{t+1,i}^{-1}).
\]
Summing over $t=0,\dots,\tau-1$ telescopes coordinatewise:
\[
\sum_{t<\tau}|E_t|
\le
GC\sum_{i=1}^d\sum_{t=0}^{\tau-1}(D_{t,i}^{-1}-D_{t+1,i}^{-1})
=
GC\sum_{i=1}^d(D_{0,i}^{-1}-D_{\tau,i}^{-1}).
\]
Finally $D_{0,i}=\sqrt{\yinit}+\epsz=\underlineD$ and $D_{\tau,i}^{-1}\ge 0$, so
\[
\sum_{t<\tau}|E_t|\le GC\sum_{i=1}^d D_{0,i}^{-1} = \frac{dGC}{\underlineD}.
\]
The second claim is immediate from $\sum_{t<\tau}E_t\le \sum_{t<\tau}|E_t|$.
\end{proof}

\subsubsection{Bounding $\sum_{t<\tau}\textcolor{cRem}{\|D_{t+1}^{-1}\tilde g_t\|^2}$}

For coordinate $i$, define
\[
a_t:=\tilde g_{t,i}^2,\qquad s_t:=y_{t,i}+\epsz^2.
\]
Then $s_{t+1}=s_t+a_t$ and $s_t\ge s_0=\yinit+\epsz^2$.

\begin{lemma}[Lagged harmonic bound]
\label{lem:harmonic}
Assume $a_t\le C^2$ for $t=0,\dots,T-1$. Then
\[
\sum_{t=0}^{T-1}\frac{a_t}{s_t}
\le
\log\Big(\frac{s_T}{s_0}\Big) + \frac{C^2}{s_0}.
\]
\end{lemma}

\begin{proof}
Identity:
$\frac{a_t}{s_t}=\frac{a_t}{s_{t+1}}+\frac{a_t^2}{s_t s_{t+1}}$.
Also $\log(s_{t+1})-\log(s_t)=\log(1+a_t/s_t)\ge \frac{a_t}{s_{t+1}}$, so
$\sum_{t=0}^{T-1}\frac{a_t}{s_{t+1}}\le \log(s_T/s_0)$.

Exact telescoping:
$\frac{1}{s_t}-\frac{1}{s_{t+1}}=\frac{a_t}{s_t s_{t+1}}$ implies
$\frac{a_t^2}{s_t s_{t+1}}=a_t(\frac{1}{s_t}-\frac{1}{s_{t+1}})\le C^2(\frac{1}{s_t}-\frac{1}{s_{t+1}})$.
Summing gives $\sum_{t=0}^{T-1}\frac{a_t^2}{s_t s_{t+1}}\le C^2(\frac{1}{s_0}-\frac{1}{s_T})\le C^2/s_0$.
\end{proof}

\begin{lemma}[Pathwise bound on $\sum_{t<\tau}\|D_t^{-1}\tilde g_t\|^2$]
\label{cor:Dtinvgt}
Let $s_0=\yinit+\epsz^2$. Then almost surely,
\[
\sum_{t<\tau}\textcolor{cRem}{\|D_t^{-1}\tilde g_t\|^2}
\le
H_T
:=
d\log\Big(1+\frac{TC^2}{s_0}\Big)+d\,\frac{C^2}{s_0}.
\]
\end{lemma}

\begin{proof}
Fix a coordinate $i$ and sum only over the active times $t=0,\dots,\tau-1$.
For such times, $(\sqrt{y_{t,i}}+\epsz)^2\ge y_{t,i}+\epsz^2=s_t$, hence
\[
\sum_{t<\tau}\frac{\tilde g_{t,i}^2}{(\sqrt{y_{t,i}}+\epsz)^2}
\le
\sum_{t=0}^{\tau-1}\frac{a_t}{s_t}
\le
\log\Big(1+\frac{\tau C^2}{s_0}\Big)+\frac{C^2}{s_0}
\le
\log\Big(1+\frac{TC^2}{s_0}\Big)+\frac{C^2}{s_0},
\]
where we used Lemma~\ref{lem:harmonic} and the bound $\tau\le T$.
Summing over coordinates $i=1,\dots,d$ yields the claim.
\end{proof}

\subsubsection{Bounding $\sum_{t<\tau}\textcolor{cMart}{M_t}$}

\begin{lemma}[Stopped martingale sum has zero mean]
\label{lem:martingale}
Define
\[
M_t := \inner{D_t^{-1}\nabla f(x_t)}{\tilde\xi_t}.
\]
Then $\EE_{t-1}[M_t]=0$ and
\[
\EE\Big[\sum_{t<\tau} M_t\Big]=0.
\]
\end{lemma}

\begin{proof}
$D_t^{-1}\nabla f(x_t)$ is $\calF_{t-1}$-measurable and $\EE_{t-1}[\tilde\xi_t]=0$, hence $\EE_{t-1}[M_t]=0$.
Because $\one_{\{t<\tau\}}$ is $\calF_{t-1}$-measurable (Definition~\ref{def:tau}),
$\EE[M_t\one_{\{t<\tau\}}]=\EE[\one_{\{t<\tau\}}\EE_{t-1}M_t]=0$.
Summing over $t=0,\dots,T-1$ yields the claim.
\end{proof}

\begin{lemma}[Freedman tradeoff for the stopped clipped-noise martingale]
\label{lem:freedman-tradeoff}
For every $\delta\in(0,1)$, with probability at least $1-\delta$,
\[
-\sum_{t<\tau} M_t
\le
\frac14\sum_{t<\tau} A_t + \Gamma\log\!\frac{1}{\delta},
\qquad
\Gamma := \frac{2V + \frac{2}{3}CG}{\underlineD}.
\]
\end{lemma}

\begin{proof}
Define the stopped martingale-difference sequence
\[
X_t := -M_t\one_{\{t<\tau\}}.
\]
Then $\EE_{t-1}[X_t]=0$ because $\one_{\{t<\tau\}}$ is $\calF_{t-1}$-measurable and Lemma~\ref{lem:martingale} gives $\EE_{t-1}[M_t]=0$.
Also,
\[
|X_t|
\le
\one_{\{t<\tau\}}\|D_t^{-1}\nabla f(x_t)\|\,\|\tilde\xi_t\|
\le
\frac{2CG}{\underlineD}
\;=:\; b.
\]
Next,
\[
\EE_{t-1}[X_t^2]
\le
\one_{\{t<\tau\}}\|D_t^{-1}\nabla f(x_t)\|^2\,\EE_{t-1}\|\tilde\xi_t\|^2.
\]
Since $\EE_{t-1}\|\tilde\xi_t\|^2 \le \EE_{t-1}\|\tilde g_t\|^2 \le \EE_{t-1}\|g_t\|^2 \le V$ and
$\|D_t^{-1}\nabla f(x_t)\|^2\le A_t/\underlineD$ by the estimate used in Lemma~\ref{lem:bias-coupled}, we obtain
\[
\EE_{t-1}[X_t^2]
\le
\frac{V}{\underlineD}A_t\one_{\{t<\tau\}}.
\]
A standard Bernstein-Freedman exponential supermartingale bound therefore yields, for every $\lambda\in(0,3/b)$,
\[
\PP\!\left(
\sum_{t=0}^{T-1} X_t
>
\frac{\lambda}{2(1-\lambda b/3)}\frac{V}{\underlineD}\sum_{t<\tau}A_t
+
\frac{\log(1/\delta)}{\lambda}
\right)
\le
\delta.
\]
Choose
\[
\lambda := \frac{\underlineD}{2V+\frac{2}{3}CG}.
\]
Then $\lambda<3/b$ and
\[
\frac{\lambda}{2(1-\lambda b/3)}\frac{V}{\underlineD} = \frac14,
\qquad
\frac{1}{\lambda} = \frac{2V+\frac{2}{3}CG}{\underlineD} = \Gamma.
\]
Substituting this choice into the previous display gives the claim.
\end{proof}

\subsubsection{High probability upper bound on $D_t$}
\label{sec:hpbbb}
\begin{lemma}[High-probability bound on $D_t$ via the total stopped energy]
\label{lem:Dhp}
Define the total stopped clipped-gradient energy
\[
\widetilde Y_T \, := \, \sum_{t=0}^{T-1}\|\tilde g_t\|^2\,\one_{\{t<\tau\}}.
\]
Then for every $\delta\in(0,1)$, with probability at least $1-\delta$,
\[
\widetilde Y_T
\le
TV + \sqrt{2TC^2V\log\!\frac{1}{\delta}} + \frac{2C^2}{3}\log\!\frac{1}{\delta}
\;=:\; B_{T,\delta}.
\]
Consequently, on the intersection $\{\tau=T\}\cap\{\widetilde Y_T\le B_{T,\delta}\}$,
\[
\max_{t\le T}\|D_t\|
\le
\sqrt{\yinit + B_{T,\delta}} + \epsz
\;=:\; \overlineD_{\delta}.
\]
\end{lemma}

\begin{proof}
Set
\[
X_t := \|\tilde g_t\|^2\one_{\{t<\tau\}},
\qquad
Z_t := X_t - \EE_{t-1}[X_t].
\]
Because $\one_{\{t<\tau\}}$ is $\calF_{t-1}$-measurable and $\|\tilde g_t\|\le C$, we have $0\le X_t\le C^2$, hence $|Z_t|\le C^2$.
Also, on $\{t<\tau\}$ we have $\|\nabla f(x_t)\|\le G$ by Lemma~\ref{lem:beforetau_grad}, so
\[
\EE_{t-1}[X_t]
\le
\EE_{t-1}[\|g_t\|^2\one_{\{t<\tau\}}]
=
\one_{\{t<\tau\}}\EE_{t-1}\|g_t\|^2
\le
\one_{\{t<\tau\}}V
\le
V.
\]
Moreover,
\[
\EE_{t-1}[Z_t^2]
\le
\EE_{t-1}[X_t^2]
\le
C^2\EE_{t-1}[X_t]
\le
C^2V.
\]
Thus the predictable quadratic variation of $\sum_{s=0}^{t} Z_s$ is bounded by $TC^2V$.
Applying Freedman's inequality to the martingale-difference sequence $(Z_t)$ with increment bound $C^2$ gives, for every $\delta\in(0,1)$,
\[
\PP\!\left(\sum_{t=0}^{T-1} Z_t > \sqrt{2TC^2V\log\!\frac{1}{\delta}} + \frac{2C^2}{3}\log\!\frac{1}{\delta}\right)
\le
\delta.
\]
Since $\sum_{t=0}^{T-1}\EE_{t-1}[X_t]\le TV$, the displayed bound implies
\[
\widetilde Y_T = \sum_{t=0}^{T-1}X_t
\le
TV + \sqrt{2TC^2V\log\!\frac{1}{\delta}} + \frac{2C^2}{3}\log\!\frac{1}{\delta}
\]
with probability at least $1-\delta$.
Finally, on $\{\tau=T\}$ we have
\[
y_{T,i} = \yinit + \sum_{t=0}^{T-1}\tilde g_{t,i}^2 \le \yinit + \sum_{t=0}^{T-1}\|\tilde g_t\|^2 = \yinit + \widetilde Y_T,
\]
and $y_t$ is coordinatewise nondecreasing, so $\max_{t\le T}y_{t,i}\le y_{T,i}$. Therefore on $\{\tau=T\}\cap\{\widetilde Y_T\le B_{T,\delta}\}$,
\[
\max_{t\le T}D_{t,i}
\le
\sqrt{\yinit+B_{T,\delta}}+\epsz
\qquad \text{for every } i,
\]
which yields the claimed operator-norm bound.
\end{proof}

\subsubsection{Proof of Lemma~\ref{lem:summed_descent}}
\label{sec:l5}
\begin{lemma}[Summed descent inequality up to \(\tau\)]
\label{lem:summed_descent_}
Assume the neighborhood condition (\ref{eq:neigh}) holds. Then, for every $\delta \in (0,1)$, with probability at least $1-\delta$, we have
\begin{equation}
\label{eq:3_}
f(x_\tau)-f^\star + \frac{h}{2}\sum_{t<\tau} \textcolor{cMain}{A_t}
\le
\Delta_0
+ \underbrace{\textcolor{cBias}{\frac{hT V^2}{\underline D C^2}}}_{\mathclap{\substack{\text{clipping-residual}\\\text{contribution}}}}
+ \quad 
\underbrace{\textcolor{cRem}{\frac{L}{2} h^2 H_T}}_{\mathclap{\substack{\text{generalized}\\\text{smoothness remainder}}}}
+ \quad 
\underbrace{h \textcolor{cSame}{\frac{dGC}{\underline D}}}_{\mathclap{\substack{\text{same-step}\\ \text{correction}}}}
+
h\textcolor{cMart}{\Gamma\log(1/\delta)},
\end{equation}
, where $H_T = d\log\left(1+\frac{TC^2}{y_{\mathrm{init}} + \varepsilon_0^2}\right) + d\frac{C^2}{y_{\mathrm{init}} + \varepsilon_0^2}$ and $\Gamma = \frac{2V+ \frac{2}{3}CG}{\underline D}$.
\end{lemma}

\begin{proof}
Summing (\ref{eq:xx}) for $t<\tau$ from Lemma~\ref{lem:onestep_} we get
\begin{equation}
\label{eq:temp1}
f(x_\tau) - f(x_0) \le -h \sum_{t<\tau} A_t - h \sum_{t<\tau} \textcolor{cBias}{\langle D_t^{-1}\nabla f(x_t), \mathbb{E}_{t-1}[r_t]\rangle} - h \sum_{t<\tau} \textcolor{cMart}{M_t} + h \sum_{t<\tau} \textcolor{cSame}{E_t} + \frac{L}{2} h^2 \sum_{t<\tau} \textcolor{cRem}{\|D_{t+1}^{-1}\tilde{g}_t\|^2}
\end{equation}

We bound the four error terms on the right-hand side separately:
\begin{itemize}
    \item \textcolor{cBias}{\textbf{Clipping Bias:}} Using Lemma~\ref{lem:bias-coupled}, $-h \sum_{t<\tau} \langle D_t^{-1}\nabla f(x_t), \mathbb{E}_{t-1}[r_t]\rangle \le \frac{h}{4} \sum_{t<\tau} A_t + \frac{h T V^2}{\underline{D} C^2}$.
    \item \textcolor{cSame}{\textbf{Same-Step Correction:}} Using Lemma~\ref{lem:samestep-nl}, $h\sum_{t<\tau} E_t \le \frac{h d GC}{\underline{D}}$.  
    \item \textcolor{cRem}{\textbf{Smoothness Remainder:}} Using Lemma~\ref{lem:harmonic}, $\frac{L}{2}h^2 \sum_{t<\tau} \|D_{t+1}^{-1}\tilde{g}_t\|^2 \le \frac{L^\circ}{2} h^2 H_T$, where $H_T = d \log(1 + \frac{T C^2}{y_{\mathrm{init}} + \epsilon_0^2}) + d \frac{C^2}{y_{\mathrm{init}} + \epsilon_0^2}$.
    \item \textcolor{cMart}{\textbf{Martingale Fluctuation:}} Using Lemma~\ref{lem:freedman-tradeoff}, $-h \sum_{t<\tau} M_t \le \frac{h}{4} \sum_{t<\tau} A_t + h \Gamma \log(1/\delta)$, where $\Gamma = \frac{2V + \frac{2}{3}C G}{\underline{D}}$.
\end{itemize}

Substituting these bounds in (\ref{eq:temp1}) and rearranging, we obtain (\ref{eq:3_}):

$$f(x_\tau) - f^* + \frac{h}{2} \sum_{t<\tau} A_t \le \Delta_0 + R_{T,\delta}^{NL}(F)$$

where $R_{T,\delta}^{NL}(F) := \frac{h T V^2}{\underline{D} C^2} + \frac{L}{2} h^2 H_T + h \Gamma \log(1/\delta) + h \frac{d GC}{\underline{D}}$.
   
\end{proof}

Now, we will prove Theorem~\ref{thm:main}.

\subsection{Proof of Theorem~\ref{thm:main}}
\label{sec:theorem1proof}
\begin{theorem}[True finite-horizon Euclidean stationarity for clipped same-step AdaGrad]
\label{thm:main_}
Suppose Assumptions~\ref{ass:1},~\ref{ass:gen}, and~\ref{ass:noise} holds. Let
\[
\Delta_0 \coloneqq f(x_0)-f^\star > 0,
\quad
F^\circ \coloneqq 2\Delta_0,
\quad
G^\circ \coloneqq G(F^\circ),
\quad
L^\circ \coloneqq \ell(2G^\circ),
\quad
V^\circ \coloneqq (G^\circ)^2+\sigma^2.
\]
Fix \(\kappa>0\), choose \((h,C)\) from~\eqref{eq:calib}, and set $y_{\mathrm{init}} \coloneqq C^2,$ $
\epsilon_0 \coloneqq C$. Suppose moreover that
\begin{equation}
\label{eq:temp_3}
\eta \le
\min\left\{
\frac{2G^\circ}{L^\circ},
\sqrt{\frac{\Delta_0}{4L^\circ d}},
\frac{\Delta_0\kappa}{8V^\circ},
\frac{\Delta_0\kappa^3}{4(V^\circ)^2},
\frac{3\Delta_0}{8G^\circ},
\frac{\Delta_0}{dG^\circ}
\right\},
\end{equation} 
then the neighborhood condition (\ref{eq:neigh}) holds at $F^\circ$. Then, with probability at least \(1-\delta\),
\[
f(x_t)-f^\star \le 2\Delta_0
\qquad \text{for all } t=0,\dots,T,
\]
and
\[
    \frac{1}{T}\!\sum_{t=0}^{T-1}\!\|\nabla f(x_t)\|^2\!
    \le\!
    \frac{4\Delta_0\sqrt{\Lambda_{T,\delta}}}{\eta\sqrt{T}}
    \!\!\left[\!
        \!\left(\!
        V^\circ
        \!+\!
        \kappa\sqrt{2V^\circ\log\frac{2}{\delta}}\,T^{-1/6}
        \!+\!
        \kappa^2\!\!\left(\!1\!+\!\frac{2}{3}\log\frac{2}{\delta}\!\right)\!T^{-1/3}\!
        \right)^{\!\!1/2}
        \!\!\!\!\!\!\!+
        \!\kappa T^{-1/6}
    \!\right]\!\!.
\]
Consequently, for fixed problem parameters,
\[
\frac{1}{T}\sum_{t=0}^{T-1}\|\nabla f(x_t)\|^2 \le \epsilon
\qquad\Longrightarrow\qquad
T = \widetilde{\mathcal O}(\epsilon^{-2}).
\]
\end{theorem}

\begin{proof}
% \paragraph{Step 1: Fixing parameters and localized descent}

We define the predictable level-set stopping time $\tau := \min\{t \le T: f(x_t) - f^* > F^\circ\} \wedge T$, where $F^\circ := 2\Delta_0$. For all $t < \tau$, the trajectory remains in the generalized-smoothness neighborhood, allowing us to apply the local descent inequality with local constants $G^\circ = G(F^\circ)$, $L^\circ = \ell(2G^\circ)$, and $V^\circ = (G^\circ)^2 + \sigma^2$.

Next, we invoke Lemma~\ref{lem:summed_descent_} under the anchored choices of parameters induced by \(F^\circ\). Then, with probability at least \(1-\delta/2\), we obtain\footnote{The equation below holds when Lemma~\ref{lem:freedman-tradeoff} is applied with parameter $\delta/2$.}:

\begin{equation}
\label{eq:temp_4}
    f(x_\tau) - f^* + \frac{h}{2} \sum_{t<\tau} A_t \le \Delta_0 + R_{T,\delta/2}^{NL}(F^\circ)
\end{equation}

where 
\begin{equation}
\label{eq:temp_2}
R_{T,\delta/2}^{NL}(F^\circ) := \frac{h T (V^\circ)^2}{\underline{D} C^2} + \frac{L}{2} h^2 H_T + h \Gamma^\circ \log(2/\delta) + h \frac{d G^\circ C}{\underline{D}}
\end{equation}.

We now apply the calibration $y_{\mathrm{init}} = C^2$ and $\epsilon_0 = C$, which ensures $\underline{D} = 2C$. We set $h = \eta / \sqrt{\Lambda_{T,\delta}}$ and $C = \kappa T^{1/3}$.

\paragraph{Step 1: Neighborhood condition.} The neighborhood condition (\ref{eq:neigh}) is
\[
\frac{hC}{\underlineD} \le \frac{G_\circ}{L_\circ}.
\]
Using $\underlineD=2C$ and $h=\eta/\sqrt{\Lambda_{T,\delta}}$, the left-hand side simplifies to
\[
\frac{hC}{\underlineD}
=
\frac{1}{2}\,\frac{\eta}{\sqrt{\Lambda_{T,\delta}}}
\le
\frac{\eta}{2}.
\]
Since $\Lambda_{T,\delta}\ge 1$, the first bound in (\ref{eq:temp_3}) implies
\[
\frac{\eta}{2}
\le
\frac{G_\circ}{L_\circ},
\]
and therefore (\ref{eq:neigh}) holds at the anchored level $F_\circ$.

\paragraph{Step 2: Exact form of $R_{T,\delta/2}(F^\circ)$.}
We now compute the four terms in
\[
R_{T,\delta/2}(F^\circ)
=
\frac{hT V_\circ^2}{\underlineD\,C^2}
+
\frac{L_\circ}{2}h^2 H_T
+
h\,\Gamma^\circ\log\frac{2}{\delta}
+
h \frac{d G^\circ C}{\underline{D}},
\]
where
\[
H_T = d\log\!\Big(1+\frac{TC^2}{y_{\mathrm{init}} + \epsilon_0}\Big)+d\,\frac{C^2}{y_{\mathrm{init}} + \epsilon_0},
\qquad
\Gamma^\circ=\frac{2V_\circ+\frac23 C G_\circ}{\underlineD}.
\]

\smallskip
\noindent\emph{First term.}
Using $\underlineD=2C$,
\[
\frac{hT V_\circ^2}{\underlineD\,C^2}
=
\frac{hT V_\circ^2}{2C^3}.
\]
Now, from (\ref{eq:calib})
\[
C^3 = \kappa^3 T,
\qquad
h = \frac{\eta}{\sqrt{\Lambda_{T,\delta}}},
\]
so the factor $T$ cancels and one gains an additional $\Lambda_{T,\delta}^{-1/2}$ improvement:
\[
\frac{hT V_\circ^2}{2C^3}
=
\frac{\eta V_\circ^2}{2\kappa^3\sqrt{\Lambda_{T,\delta}}}
\le
\frac{\eta V_\circ^2}{2\kappa^3}.
\]
Thus the bias term is uniformly bounded and is in fact smaller than before by the factor $\Lambda_{T,\delta}^{-1/2}$.

\smallskip
\noindent\emph{Second term.}
Let $s_0 = y_{\mathrm{init}} + \epsilon_0^2$, then $s_0=2C^2$,
\[
\frac{TC^2}{s_0} = \frac{T}{2},
\qquad
\frac{C^2}{s_0} = \frac12.
\]
Therefore
\[
H_T
=
d\log\!\Big(1+\frac{T}{2}\Big)+\frac{d}{2}.
\]
Multiplying by $h^2=\eta^2/\Lambda_{T,\delta}$ gives
\[
\frac{L_\circ}{2}h^2H_T
=
\frac{L_\circ\eta^2}{2\Lambda_{T,\delta}}
\left[d\log\!\Big(1+\frac{T}{2}\Big)+\frac{d}{2}\right].
\]
Because
\[
\log\!\Big(1+\frac{T}{2}\Big) \le \log(1+T)
\qquad\text{and}\qquad
\frac12\le 1,
\]
we have
\[
d\log\!\Big(1+\frac{T}{2}\Big)+\frac{d}{2}
\le
d\bigl(\log(1+T)+1\bigr)
\le d\,\Lambda_{T,\delta}.
\]
Hence
\[
\frac{L_\circ}{2}h^2H_T
\le
\frac{L_\circ d\eta^2}{2}.
\]

\smallskip
\noindent\emph{Third term.}
We denote $\lambda_\delta = \log\frac{2}{\delta}$ for simplicity. Since $\underlineD=2C$, and ,
\[
\Gamma^\circ
=
\frac{2V_\circ+\frac23 C G_\circ}{2C}
=
\frac{V_\circ}{C} + \frac{G_\circ}{3}.
\]
Therefore
\[
h\Gamma^\circ\lambda_\delta
=
\frac{\eta\lambda_\delta}{\sqrt{\Lambda_{T,\delta}}}
\left(\frac{V_\circ}{C}+\frac{G_\circ}{3}\right).
\]
We bound the two pieces separately.
For the $V_\circ/C$ piece, the calibration gives
\[
\frac{1}{C}
=
\frac{1}{\kappa T^{1/3}},
\]
so
\[
\frac{\eta\lambda_\delta}{\sqrt{\Lambda_{T,\delta}}}\cdot\frac{V_\circ}{C}
=
\frac{\eta V_\circ\lambda_\delta}{\kappa T^{1/3}\sqrt{\Lambda_{T,\delta}}}.
\]
Now $T\ge 1$ implies $T^{-1/3}\le 1$, and $\lambda_\delta\le \sqrt{\Lambda_{T,\delta}}$ by definition of $\Lambda_{T,\delta}$. Therefore
\[
\frac{\lambda_\delta}{T^{1/3}\sqrt{\Lambda_{T,\delta}}}
\le
1,
\]
and hence
\[
\frac{\eta\lambda_\delta}{\sqrt{\Lambda_{T,\delta}}}\cdot\frac{V_\circ}{C}
\le
\frac{\eta V_\circ}{\kappa}.
\]
For the $G_\circ/3$ piece, again using $\lambda_\delta\le\sqrt{\Lambda_{T,\delta}}$,
\[
\frac{\eta\lambda_\delta}{\sqrt{\Lambda_{T,\delta}}}\cdot\frac{G_\circ}{3}
\le
\frac{\eta G_\circ}{3}.
\]
Combining the two estimates yields
\[
h\Gamma^\circ\lambda_\delta
\le
\frac{\eta V_\circ}{\kappa} + \frac{\eta G_\circ}{3}.
\]

\smallskip
\noindent\emph{Third term.}
Because $\underlineD=2C$,
\[
h\,\frac{dG_\circ C}{\underlineD}
=
\frac{h dG_\circ}{2}
=
\frac{\eta dG_\circ}{2\sqrt{\Lambda_{T,\delta}}}
\le
\frac{\eta dG_\circ}{2},
\]
where we again used $\Lambda_{T,\delta}\ge1$.
\paragraph{Step 3: Global bound on $R_{T,\delta/2}(F^\circ)$.}
Summing the four pieces established above gives
\[
R_{T,\delta/2}(F^\circ)
\le
\frac{\eta V_\circ^2}{2\kappa^3}
+
\frac{L_\circ d\eta^2}{2}
+
\frac{\eta V_\circ}{\kappa}
+
\frac{\eta G_\circ}{3}
+
\frac{\eta dG_\circ}{2}.
\]
Now each term is controlled by one of the explicit bounds in \eqref{eq:temp_3}:
\[
\frac{\eta V_\circ^2}{2\kappa^3} \le \frac{\Delta_0}{8},
\qquad
\frac{L_\circ d\eta^2}{2} \le \frac{\Delta_0}{8},
\qquad
\frac{\eta V_\circ}{\kappa} \le \frac{\Delta_0}{8},
\qquad
\frac{\eta G_\circ}{3} \le \frac{\Delta_0}{8},
\qquad
\frac{\eta dG_\circ}{2} \le \frac{\Delta_0}{2}
.
\]
Summing these upper bounds yeilds
\[
R_{T,\delta/2}(F^\circ)
\le
\Delta_0.
\]
Recalling that $F_\circ=2\Delta_0$, we conclude that
\[
F_\circ = 2\Delta_0 \ge \Delta_0 + R_{T,\delta/2}(F_\circ),
\]

Consequently, (\ref{eq:temp_4}) simplifies to $f(x_\tau) - f^* \le 2\Delta_0 = F^\circ$. By definition of the stopping time, this strictly rules out $\tau < T$, proving that $\tau = T$ (the trajectory never exits the level set) and yielding:

\begin{equation}
\label{eq:temp_5}
\frac{1}{T}\sum_{t=0}^{T-1} A_t \le \frac{2F^\circ}{h T} = \frac{4\Delta_0}{h T}
\end{equation}
with probability $1-\delta/2$.

\paragraph{Step 4: Euclidean Conversion and Final Complexity.}

To convert the preconditioned bound $A_t$ into an unconditioned Euclidean bound, we must formalize the intersection of two high-probability events.

First, let $\mathcal{E}_{pre}$ denote the event that our preconditioned descent bound from Step 4 holds, i.e (\ref{eq:temp_5}). Specifically,
\[
\mathcal{E}_{pre}
:=
\left\{
\tau=T
\ \text{ and }\ 
\frac{1}{T}\sum_{t=0}^{T-1}A_t\le \frac{4\Delta_0}{hT}
\right\}.
\]
Because this bound relies on the Bernstein Freedman supermartingale inequality applied with a confidence parameter of $\delta/2$\footnote{This holds when Lemma~\ref{lem:freedman-tradeoff} is applied with parameter $\delta/2$.}, we have $\mathbb{P}(\mathcal{E}_{pre}) \ge 1 - \delta/2$. On this event, $\tau = T$ and $\frac{1}{T}\sum_{t=0}^{T-1} A_t \le \frac{4\Delta_0}{h T}$.

Second, let $\mathcal{E}_{den}$ denote the denominator concentration event, i.e $
\mathcal{E}_{den} := \{\widetilde Y_T\le B_{T,\delta/2}\}$. By Lemma ~\ref{lem:Dhp}, applied to the stopped clipped-gradient energy $\tilde{Y}_T$ with a confidence parameter of $\delta/2$, the total energy is bounded by:$$B_{T,\delta/2} := T V^\circ + \sqrt{2T C^2 V^\circ \log(2/\delta)} + \frac{2C^2}{3}\log(2/\delta)$$This event also holds with probability $\mathbb{P}(\mathcal{E}_{den}) \ge 1 - \delta/2$.

To ensure both bounds hold simultaneously, we apply the union bound. The probability that either event fails is at most $\delta/2 + \delta/2 = \delta$. Therefore, the probability that both events hold is:$$\mathbb{P}(\mathcal{E}_{pre} \cap \mathcal{E}_{den}) \ge 1 - \delta$$We proceed on this intersection. On $\mathcal{E}_{den}$, the maximum denominator norm is deterministically bounded for all $t \le T$:$$\|D_t\| \le \overline{D}_\delta := \sqrt{y_{\mathrm{init}} + B_{T,\delta/2}} + \epsilon_0$$Because the diagonal metric satisfies $D_t \le \overline{D}_\delta I$, its inverse satisfies $D_t^{-1} \ge \overline{D}_\delta^{-1} I$. This allows us to lower-bound the preconditioned gradient progress:$$A_t = \langle\nabla f(x_t), D_t^{-1}\nabla f(x_t)\rangle \ge \overline{D}_\delta^{-1} \|\nabla f(x_t)\|^2$$Averaging this inequality over the horizon $T$ and applying the bound from $\mathcal{E}_{pre}$ yields:$$\frac{1}{T} \sum_{t=0}^{T-1} \|\nabla f(x_t)\|^2 \le \overline{D}_\delta \left( \frac{1}{T} \sum_{t=0}^{T-1} A_t \right) \le \overline{D}_\delta \frac{4\Delta_0}{h T}$$To isolate the explicit dependence on the horizon $T$, we expand $\overline{D}_\delta$ using our parameter calibrations $y_{\mathrm{init}} = C^2$, $\epsilon_0 = C$, and $C = \kappa T^{1/3}$:$$\overline{D}_\delta = \sqrt{C^2 + T V^\circ + \sqrt{2T C^2 V^\circ \log(2/\delta)} + \frac{2C^2}{3}\log(2/\delta)} + C$$We factor out $\sqrt{T}$ from both terms. Noting that $C^2 / T = \kappa^2 T^{-1/3}$ and $C / \sqrt{T} = \kappa T^{-1/6}$, we obtain:$$\overline{D}_\delta \le \sqrt{T} \left[ \left( V^\circ + \kappa \sqrt{2V^\circ \log(2/\delta)} T^{-1/6} + \kappa^2\left(1 + \frac{2}{3}\log(2/\delta)\right) T^{-1/3} \right)^{1/2} + \kappa T^{-1/6} \right]$$Finally, we substitute this expanded $\overline{D}_\delta$ and the calibrated step size $h = \eta / \sqrt{\Lambda_{T,\delta}}$ back into the Euclidean bound. The $\sqrt{T}$ from the denominator bound and the $1/T$ from the averaging cleanly collapse:$$\overline{D}_\delta \frac{4\Delta_0}{h T} = \sqrt{T}[\dots] \cdot \frac{4\Delta_0 \sqrt{\Lambda_{T,\delta}}}{\eta T} = \frac{4\Delta_0 \sqrt{\Lambda_{T,\delta}}}{\eta \sqrt{T}} [\dots]$$This simplifies exactly to the finite-horizon theorem guarantee. Because the bracketed term limits to $\sqrt{V^\circ}$ as $T \to \infty$ and $\Lambda_{T,\delta}$ contributes only polylogarithmic factors, this directly validates the $\tilde{\mathcal{O}}(T^{-1/2})$ convergence rate.

The actual rate is:$$ \mathcal{O}\left( \frac{d\big(\sqrt{\log T}+ \log(1/\delta)\big)}{\sqrt{T}} \right) $$.

\subsubsection{Discussion on specific choice of $F^\circ = 2\Delta_0$}
\label{sec:appex}
While Theorem~\ref{thm:main} explicitly anchors the analysis at the level $F^\circ = 2\Delta_0$, this choice is an analytic convenience rather than an algorithmic limitation. The $\tilde{\mathcal{O}}(T^{-1/2})$ Euclidean stationarity guarantee natively extends to any arbitrary target level $F > \Delta_0$. For any such choice, the induced local generalized-smoothness constants $G(F)$ and $L(F)$ remain finite and fixed, yielding a strictly positive feasibility slack $F - \Delta_0 > 0$. Because the cumulative algorithmic and stochastic penalties, specifically the clipping-residual bias, the deterministic same-step correction, the martingale fluctuation, and the smoothness remainder, scale proportionately with the hyperparameters, there always exists a sufficiently small step-size parameter $\eta$ capable of compressing the total error budget entirely within this available slack.We specifically highlight $F^\circ = 2\Delta_0$ because it constructs an optimal, symmetric $1:1$ feasibility partition: it dedicates exactly $\Delta_0$ to the initial suboptimality and leaves an equal maximum allowance of $\Delta_0$ for the noise budget. To understand why this balance is critical, consider the mechanical tradeoff in the final Euclidean guarantee. Selecting a substantially larger target level (e.g., $F \gg \Delta_0$) artificially widens the error budget, which permissibly allows for a larger step-size parameter $\eta$. Since $\eta$ dictates the denominator of the final complexity bound, increasing it nominally improves the theoretical convergence rate. However, under generalized smoothness, the localized gradient bound $G(F)$ and smoothness parameter $L(F)$ grow monotonically, and potentially rapidly, with the level-set threshold $F$. Because these geometric constants dictate the numerator of the final bound (acting through the variance factor $V(F)$ and the denominator bound $\overline{D}_\delta$), a larger $F$ causes the numerator to inflate significantly faster than the step size improves. Consequently, an overly large $F$ severely degrades the leading constants hidden within the asymptotic rate, rendering the theoretical upper bound loose and uninformative. Thus, anchoring at $2\Delta_0$ serves as an elegant structural optimal: it provides exactly enough slack to ensure feasibility without unnecessarily deteriorating the problem's local geometry.

Conversely, while selecting a tighter target level $\Delta_0 < F < 2\Delta_0$ would marginally improve the local curvature constants, it would severely shrink the available feasibility slack. This would force a proportionally microscopic step size $\eta$, unnecessarily inflating the iteration complexity.
% Under this calibration, (\ref{eq:temp_2}) strictly evaluates to five components:

% \[R_{T,\delta/2}^{NL}(F^\circ) \le \frac{\eta (V^\circ)^2}{2\kappa^3} + \frac{L^\circ d\eta^2}{2} + \frac{\eta V^\circ}{\kappa} + \frac{\eta G^\circ}{3} + \frac{\eta d G^\circ}{2}\]
\end{proof}
% \begin{lemma}[Deterministic bound on the cumulative same-step correction]
% \label{lem:samestep_}
% For every sample path, $\sum_{t<\tau} |\textcolor{cSame}{E_t}| \le \frac{dGC}{\underline D}.
% $ Consequently,$\sum_{t<\tau} \textcolor{cSame}{E_t} \le \frac{dGC}{\underline D}.
% $
% \end{lemma}
% \begin{proof}
% Fix a coordinate $i$.
% Because $y_{t+1,i}\ge y_{t,i}$, the scalar sequence $(D_{t,i}^{-1})_{t\ge 0}$ is nonincreasing, so
% \[
% D_{t,i}^{-1}-D_{t+1,i}^{-1}\ge 0.
% \]
% For $t<\tau$, Lemma~\ref{lem:beforetau_grad} gives $|\nabla_i f(x_t)|\le G$, while clipping gives $|\tilde g_{t,i}|\le C$.
% Hence
% \[
% |E_t|
% \le
% \sum_{i=1}^d |\nabla_i f(x_t)|\,(D_{t,i}^{-1}-D_{t+1,i}^{-1})\,|\tilde g_{t,i}|
% \le
% GC\sum_{i=1}^d (D_{t,i}^{-1}-D_{t+1,i}^{-1}).
% \]
% Summing over $t=0,\dots,\tau-1$ telescopes coordinatewise:
% \[
% \sum_{t<\tau}|E_t|
% \le
% GC\sum_{i=1}^d\sum_{t=0}^{\tau-1}(D_{t,i}^{-1}-D_{t+1,i}^{-1})
% =
% GC\sum_{i=1}^d(D_{0,i}^{-1}-D_{\tau,i}^{-1}).
% \]
% Finally $D_{0,i}=\sqrt{\yinit}+\epsz=\underlineD$ and $D_{\tau,i}^{-1}\ge 0$, so
% \[
% \sum_{t<\tau}|E_t|\le GC\sum_{i=1}^d D_{0,i}^{-1} = \frac{dGC}{\underlineD}.
% \]
% The second claim is immediate from $\sum_{t<\tau}E_t\le \sum_{t<\tau}|E_t|$.
% \end{proof}

\end{document}